\documentclass[11pt]{amsart}

\usepackage{eulervm}
\usepackage{eufrak}

\usepackage[usenames,dvipsnames,svgnames,table]{xcolor}
\usepackage{amssymb,amsfonts, amsrefs}
\usepackage{amsmath,amsthm}
\usepackage{pgf, tikz}
\usetikzlibrary{arrows.meta, positioning, shapes.geometric, calc, decorations.pathreplacing}
\usepackage{enumitem}
\usepackage{graphicx}
\usepackage{booktabs}
\usepackage{dsfont}
\usepackage[pdfdisplaydoctitle,colorlinks,urlcolor=blue,linkcolor=blue,citecolor=blue]{hyperref}
\usepackage{algorithm}
\usepackage{algorithmic}

\newcommand{\e}{\epsilon}
\renewcommand{\d}{\delta}
\renewcommand{\a}{\alpha}
\renewcommand{\b}{\beta}

\newcommand\1{{\mathds{1}}}
\newcommand\PP{{\mathbb{P}}}
\newcommand\E{{\mathbb{E}}}
\newcommand\N{{\mathbb{N}}}
\newcommand\D{{\mathcal{D}}}

\newcommand\A{{\mathcal{A}}}

\newcommand\C{{\mathcal{C}}}

\newcommand\X{{\mathcal{X}}}

\DeclareMathOperator{\Var}{Var}

\DeclareMathOperator{\MSE}{MSE}
\DeclareMathOperator{\Bias}{Bias}

\newtheorem{theorem}{Theorem}[section]
\newtheorem{proposition}[theorem]{Proposition}

\newtheorem{corollary}[theorem]{Corollary}
\theoremstyle{definition}
\newtheorem{definition}[theorem]{Definition}

\newtheorem{assumption}[theorem]{Assumption}
\theoremstyle{remark}
\newtheorem{remark}[theorem]{Remark}
\numberwithin{equation}{section}

\begin{document}

\title[Reliability Certification for AI Agents]
{Black-Box Reliability Certification for AI Agents via\\Self-Consistency Sampling and Conformal Calibration}

\author{Charafeddine Mouzouni}

\address{OPIT -- Open Institute of Technology, and Cohorte AI, Paris, France.}
\email{charafeddine{@}cohorte.co}
\dedicatory{\small\textit{OPIT -- Open Institute of Technology, and Cohorte AI, Paris, France.}\\[2pt]\texttt{charafeddine@cohorte.co}}

\date{\today}

\begin{abstract}
Given a black-box AI system and a task, at what confidence level can a practitioner trust the system's output? We answer with a \emph{reliability level}---a single number per system--task pair, derived from \emph{self-consistency sampling} and \emph{conformal calibration}, that serves as a black-box deployment gate with \emph{exact, finite-sample, distribution-free} guarantees. Self-consistency sampling reduces uncertainty exponentially; conformal calibration guarantees correctness within $1/(n{+}1)$ of the target level, regardless of the system's errors---made \emph{transparently visible} through larger answer sets for harder questions. Weaker models earn lower reliability levels (not accuracy---see Definition~\ref{def:reliability}): GPT-4.1 earns $94.6\%$ on GSM8K and $96.8\%$ on TruthfulQA, while GPT-4.1-nano earns $89.8\%$ on GSM8K and $66.5\%$ on MMLU. We validate across five benchmarks, five models from three families, and both synthetic and real data. Conditional coverage on solvable items exceeds $0.93$ across all configurations; sequential stopping reduces API costs by ${\sim}50\%$.
\end{abstract}

\keywords{reliability certification, conformal prediction, self-consistency, LLM evaluation, deployment gating, uncertainty quantification.}

\maketitle

\section{Introduction}\label{sec:intro}

\begin{figure}[t]
\centering
\begin{tikzpicture}[
    >=Stealth,
    stepbox/.style={
        rectangle, rounded corners=5pt, draw=blue!40, line width=0.9pt,
        fill=blue!4, minimum width=3.8cm, minimum height=3.4cm,
        inner sep=6pt, align=center
    },
    steplabel/.style={
        font=\sffamily\bfseries\small, text=blue!70!black
    },
    annot/.style={font=\small, text=black!70, align=center},
    arrowstyle/.style={->, line width=2pt, color=blue!40!black,
        shorten <=2pt, shorten >=2pt}
]

\node[stepbox] (sample) {%
    \raisebox{0pt}[0pt][0pt]{%
    \begin{minipage}{3.4cm}\centering
    \vspace*{4pt}
    {\small\sffamily\bfseries\color{blue!60!black} AI System}\\[6pt]
    {\footnotesize Question $\to$ $K{=}10$ calls}\\[8pt]
    \setlength{\tabcolsep}{1pt}
    {\footnotesize\ttfamily
    \begin{tabular}{ccccc}
    \colorbox{blue!15}{\strut\,42} &
    \colorbox{blue!15}{\strut\,42} &
    \colorbox{blue!15}{\strut\,42} &
    \colorbox{orange!15}{\strut\,37} &
    \colorbox{blue!15}{\strut\,42} \\[3pt]
    \colorbox{orange!15}{\strut\,37} &
    \colorbox{blue!15}{\strut\,42} &
    \colorbox{blue!15}{\strut\,42} &
    \colorbox{blue!15}{\strut\,42} &
    \colorbox{blue!15}{\strut\,42}
    \end{tabular}}
    \vspace*{2pt}
    \end{minipage}}
};
\node[steplabel, above=3pt of sample] {\textsc{1.\;Sample}};
\node[annot, below=3pt of sample] {Ask the same question\\$K$ times};

\node[stepbox, right=0.35cm of sample] (rank) {%
    \raisebox{0pt}[0pt][0pt]{%
    \begin{minipage}{3.6cm}\centering
    \vspace*{4pt}
    {\footnotesize\sffamily Group \& count}\\[8pt]
    {\small\ttfamily
    \colorbox{blue!25}{\;\textbf{``42''}\;$\to$\;\textbf{8}\,/\,10\;}\\[5pt]
    \colorbox{orange!12}{\;``37''\;$\to$\;2\,/\,10\;}}\\[8pt]
    {\footnotesize\sffamily Ranked by frequency}
    \vspace*{2pt}
    \end{minipage}}
};
\node[steplabel, above=3pt of rank] {\textsc{2.\;Rank}};
\node[annot, below=3pt of rank] {Sort identical answers\\by frequency};

\node[stepbox, right=0.35cm of rank] (calib) {%
    \raisebox{0pt}[0pt][0pt]{%
    \begin{minipage}{3.6cm}\centering
    \vspace*{4pt}
    {\footnotesize\sffamily Human checks $n{\approx}50$ items}\\[8pt]
    {\scriptsize\sffamily
    \begin{tabular}{@{}l@{}}
    \color{green!50!black}$\checkmark$\; item 1: top 1 correct\\[2pt]
    \color{green!50!black}$\checkmark$\; item 2: top 1 correct\\[2pt]
    \color{red!60!black}$\times$\; item 3: need top 2
    \end{tabular}}\\[8pt]
    {\small$\longrightarrow$\;
    \fboxsep=3pt\fcolorbox{blue!50}{blue!50}{\color{white}\bfseries\sffamily 94.6\,\%}}
    \vspace*{2pt}
    \end{minipage}}
};
\node[steplabel, above=3pt of calib] {\textsc{3.\;Calibrate}};
\node[annot, below=3pt of calib] {Small batch $\to$\\reliability level};

\draw[arrowstyle] (sample.east) -- (rank.west);
\draw[arrowstyle] (rank.east) -- (calib.west);

\end{tikzpicture}
\caption{Pipeline overview.
\textbf{Step~1}: ask the AI system the same question $K$ times and collect its answers.
\textbf{Step~2}: group identical answers and rank them by frequency.
\textbf{Step~3}: a human checks a small calibration batch; the framework outputs a single \emph{reliability level} (e.g.\ $94.6\%$) with a formal coverage guarantee.
No model internals are needed---only API access.}
\label{fig:pipeline}
\end{figure}
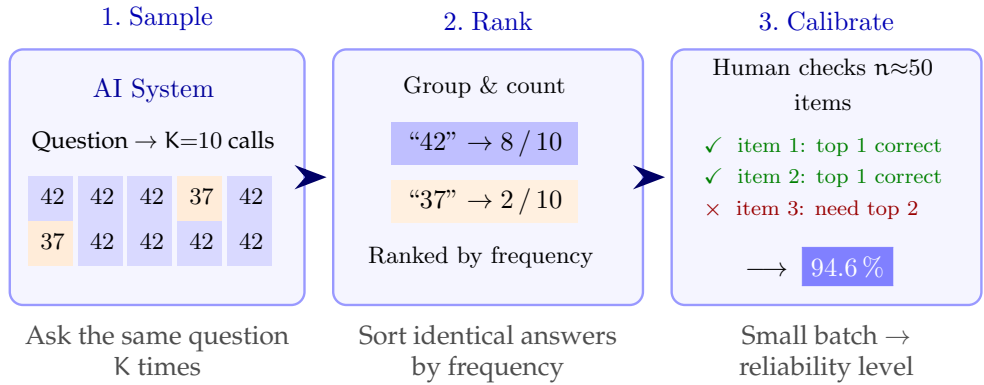

You have an AI system and a task---say, answering math questions or triaging support tickets. Before you deploy it, you need to know: \emph{how much can I trust this system?} Not a vague intuition, but a concrete number with a guarantee attached. That is what this paper provides.

The idea is simple (Figure~\ref{fig:pipeline}). For each test question, ask the AI system the same question $K$ times (say, $K{=}10$). Group the answers that say the same thing and rank them by frequency: the most popular answer might appear $8$ out of $10$ times, the runner-up $2$ out of $10$. This frequency ranking is the raw signal---it captures how ``sure'' the system is without ever looking inside its weights. Then, a human spot-checks a small random batch---just $50{-}100$ items. Each check takes seconds: the human sees the question and the system's top-ranked answer, and marks it right or wrong. From these quick judgments, the framework computes a single number: the \emph{reliability level}. For instance, GPT-4.1 earns $94.6\%$ reliability on grade-school math.

The reliability level comes with a formal guarantee: it is a valid coverage bound that holds regardless of the AI system's systematic errors, and it requires no assumptions about the data distribution. Crucially, the human is verifying, not labeling: the AI generates and ranks the answers; the human just confirms or rejects the top pick. There is no need to build a gold-standard dataset---fifty quick judgments are the entire human effort. A weaker system earns a lower number, not a misleading score; the framework makes trustworthiness \emph{transparent}.

\medskip
\noindent
More precisely, current evaluation methods each fail in a different way.
\emph{Single-sample evaluation} is unbiased but high-variance---a noisy snapshot of the agent's true capability.
\emph{Naive self-consistency} (mode selection) reduces variance but can \emph{amplify} bias: when the agent's most frequent answer is wrong, more samples make the wrong answer look \emph{more} certain.
\emph{LLM-as-judge} approaches \cite{Zheng2023Judge} layer poorly characterized biases---position bias, verbosity preference, self-preference---on top of the agent's own errors, and provide no formal reliability guarantees.
None of these methods simultaneously controls both variance and bias.

We introduce the \emph{reliability level} (Definition~\ref{def:reliability})---a single number per agent--task pair that answers this question with provable guarantees. The reliability level is a black-box deployment gate: it requires only API access, no model internals, and its validity is distribution-free with finite-sample coverage guarantees. Concretely, GPT-4.1 achieves 94.6\% reliability on GSM8K and 96.8\% on TruthfulQA; GPT-4.1-nano achieves 89.8\% on GSM8K and 66.5\% on MMLU. Open-weight models (Llama~4 Maverick, Mistral Small~24B) range from 66.7\% (MMLU) to 95.4\% (GSM8K) across three benchmarks (Table~\ref{tab:reliability}). A practitioner can read these numbers directly: ``we need $X$\% reliability---which models qualify?''

The mechanism behind the reliability level combines two ingredients. Self-consistency sampling reduces \textbf{variance} exponentially via aggregation (Theorem~\ref{thm:variance_reduction}). Conformal calibration provides coverage guarantees whose \textbf{validity is independent of agent bias}: the evaluation's coverage statement is correct regardless of the agent's systematic error profile (Theorem~\ref{thm:bias_immunity}). Agent bias is not hidden---it is made \emph{transparently visible} through larger prediction sets (Theorem~\ref{thm:bias_transparency}). A weaker agent earns a lower reliability level, not a misleading score.

A key consequence---and a diagnostic that distinguishes calibration failure from agent limitation---is that marginal coverage can fall below $1{-}\a$ \emph{only} when the agent cannot solve certain items at all. This under-coverage is \emph{predicted} by the theory, not a calibration artifact: conditional coverage on items the model can solve remains near-perfect ($\ge 0.93$ across all models and benchmarks in our experiments). When observed coverage falls short, the framework identifies \emph{why}: the gap equals the fraction of unsolvable items.

\subsection*{Contributions}
Self-consistency sampling \cite{Wang2023SelfConsistency} and conformal prediction \cite{Vovk2005, Angelopoulos2021} are individually known. Recent work \cite{Quach2023ConformalLM, Kumar2023ConformalNLP} applies conformal methods to language models using internal logits or softmax probabilities. Our contribution is a specific synthesis: a conformal score built entirely from \emph{external} sample frequencies, producing a single reliability number for deployment gating---with no access to model internals. Specifically:
\begin{enumerate}[leftmargin=1.2em]
\item \textbf{Reliability certification for deployment gating} (Definition~\ref{def:reliability}, Table~\ref{tab:reliability}): the primary practical output---a single reliability level per agent--task pair that answers ``at what confidence can I trust this agent?'' Validated across five benchmarks, five models from three families (GPT-4.1 ladder, Llama~4 Maverick, Mistral Small), with reliability levels ranging from $66.5\%$ to $96.8\%$.
\item \textbf{A black-box nonconformity score from ranked canonical consensus} (Sections~\ref{sec:selfconsist}--\ref{sec:conformal}): the technical construction enabling the reliability level---conformal scores from the rank of the acceptable answer in the self-consistency ordering, leveraging the variance reduction of aggregation while inheriting the distribution-free validity of conformal prediction. This specific construction and its analysis (variance reduction theorems, canonicalization-induced amplification) are new.
\item \textbf{Reliability theorems} (Section~\ref{sec:reliability}): the evaluation's coverage error is bounded by $1/(n+1)$ regardless of the agent's bias profile (Theorem~\ref{thm:bias_immunity}); prediction set size is a monotone, transparent diagnostic of agent quality (Theorem~\ref{thm:bias_transparency}); and the method achieves lower coverage error than LLM-as-judge evaluation once the calibration set exceeds $\lceil 1/|b_J| \rceil - 1$ (e.g., $n \ge 19$ for typical judge bias; Corollary~\ref{cor:dominance}), though the two approaches answer complementary questions (Remark~\ref{rem:complementary}).
\item \textbf{Bias--variance anatomy of LLM evaluation} (Section~\ref{sec:biasvar}): the motivating analysis---a formal decomposition showing that single-sample evaluation suffers from variance, LLM-as-judge from irreducible bias, and naive self-consistency from bias amplification.
\item \textbf{Sequential sampling with certified early stopping} (Section~\ref{sec:sequential}): a Hoeffding-based stopping rule that reduces API cost by ${\sim}50\%$ with no loss in coverage.
\end{enumerate}

\subsection*{Related Work}\label{sec:related}

\paragraph{Self-consistency decoding.}
Wang et al.\ \cite{Wang2023SelfConsistency} introduced self-consistency as an inference-time strategy: sample multiple reasoning traces and select the most frequent final answer. We repurpose self-consistency for \emph{evaluation}: the ranked consensus provides raw material for conformal calibration. Cordero-Encinar and Duncan \cite{CorderoEncinar2025} certify when the majority-vote answer has stabilized; our stopping criterion (Theorem~\ref{thm:sequential_hoeffding}) certifies the \emph{ranking quality} of the top-$M$ candidates for conformal set construction. The two are complementary.

\paragraph{Conformal prediction and language models.}
Conformal prediction \cite{Vovk2005, Angelopoulos2021} provides prediction sets that are distribution-free with finite-sample coverage guarantees. Recent work applies conformal methods to language models: Quach et al.\ \cite{Quach2023ConformalLM} construct conformal sets over token sequences for language generation; Kumar et al.\ \cite{Kumar2023ConformalNLP} apply conformal prediction to NLP classification tasks. These methods define nonconformity scores from model logits or softmax probabilities---internal quantities unavailable for black-box API-based agents. Our score is constructed entirely from \emph{external} ranked consensus over multiple samples, requiring no access to model internals. This black-box property is essential for evaluating proprietary LLM APIs.

\paragraph{Method comparison.}
Quach et al.\ and Kumar et al.\ achieve tighter prediction sets when softmax probabilities are available, because internal scores carry richer information than sample frequencies. Our contribution is \emph{generality}: the framework applies to any black-box agent accessible only through sampling---commercial APIs, tool-using agents, multi-step pipelines---and to any task where canonicalization is feasible, including open-ended generation where token-level conformal methods do not apply.

\paragraph{LLM-as-judge and evaluation bias.}
Zheng et al.\ \cite{Zheng2023Judge} established the LLM-as-judge paradigm. Subsequent studies documented systematic biases: position bias, verbosity bias, self-preference, and anchoring effects \cite{Zheng2023Judge}. Our bias--variance analysis (Section~\ref{sec:biasvar}) formalizes these observations: judge bias is an irreducible MSE component that does not decay with sample size (Proposition~\ref{prop:judge_biasvar}). Our framework avoids judge bias entirely by using frequency-based ranking rather than quality scoring.

\paragraph{Uncertainty quantification for LLMs.}
SelfCheckGPT \cite{MankulSelfCheckGPT} detects hallucinations via inter-sample agreement. Semantic entropy \cite{Kuhn2023SemanticEntropy} clusters generations and computes entropy over semantic classes. ConU \cite{ConU2024} applies conformal prediction using token-level probabilities. Although some APIs expose logprobs, ConU's conformal guarantee requires \emph{calibrated} token-level probabilities---a condition that is unverifiable for proprietary systems and unreliable even when logprobs are nominally available. Our rank-based nonconformity score is the first to provide formal conformal guarantees without any model internals, enabling calibration for any system accessible only through sampling. All three prior methods produce uncertainty \emph{estimates} without formal calibration guarantees. Our framework produces calibrated prediction \emph{sets} with finite-sample coverage guarantees, and uniquely pairs this black-box score with a deployment-gating output---the reliability level---a single actionable number that no prior conformal-LLM method provides. Semantic entropy could serve as a complementary pre-filter within conformal calibration.

\paragraph{Calibration of probabilistic predictions.}
Classical calibration methods (Platt scaling, temperature scaling, isotonic regression) adjust model confidence scores to match empirical frequencies. These require access to model probabilities and assume a fixed model; they cannot handle open-ended generation where the output space is combinatorial. Conformal prediction, by contrast, is model-agnostic and makes no distributional assumptions, which is why we adopt it as the calibration backbone.

\section{Problem Setting}\label{sec:setup}

We formalize the setting: an AI system receives a query, produces a stochastic answer, and a task-specific predicate decides whether that answer is acceptable. The goal is to certify, from a small calibration set, how reliably the system's top-ranked answer is acceptable. Our experiments validate the framework on single-turn query-answering systems; extensions to multi-turn interactions are discussed in Section~\ref{sec:limitations}.

\subsection{Queries, answers, and acceptability}
Let $\X$ denote the space of queries and $\A$ the space of possible answers. An AI system with parameters $\theta$ defines a stochastic mapping
\[
f_\theta: \X \to \A, \qquad f_\theta(x) \sim P_\theta(\cdot \mid x),
\]
where $P_\theta(\cdot \mid x)$ is the system's output distribution conditional on query $x$, accessible via repeated sampling.

\begin{definition}[Acceptability]\label{def:acceptable}
For a query $x\in\X$, the set of \emph{acceptable answers} is
\[
\A^\star(x) := \{a \in \A : \mathrm{Accept}(x, a) = 1\},
\]
where $\mathrm{Accept}: \X \times \A \to \{0,1\}$ is a task-dependent predicate. We allow $|\A^\star(x)| \ge 1$. For example, on a math problem with answer $42$, $\mathrm{Accept}$ returns $1$ for ``$42$'', ``$42.0$'', and ``forty-two'', and $0$ for ``$43$''.
\end{definition}

\begin{assumption}[Non-triviality]\label{ass:nontrivial}
For each query $x$ in the target distribution, $\A^\star(x) \neq \emptyset$.
\end{assumption}

\subsection{Per-query acceptability rate and agent quality}

The central quantity is the probability that a single random sample from the agent is acceptable---intuitively, how often the agent ``gets it right'' on a given question.

\begin{definition}[Per-query acceptability rate]\label{def:pstar}
For a fixed query $x$, the agent's \emph{acceptability rate} is
\begin{equation}\label{eq:pstar}
p^\star(x) := P_\theta\bigl(f_\theta(x) \in \A^\star(x)\bigr) = \sum_{a \in \A^\star(x)} P_\theta(a \mid x).
\end{equation}
The \emph{aggregate accuracy} of the agent over a query distribution $\mu$ on $\X$ is
\begin{equation}\label{eq:aggregate_acc}
\bar{p} := \E_{X \sim \mu}[p^\star(X)].
\end{equation}
\end{definition}

The quantities $p^\star(x)$ and $\bar{p}$ are the ground truth that any evaluation method seeks to estimate or certify. Our framework does not estimate $\bar{p}$ directly; instead, it provides a \emph{coverage guarantee} that is a stronger, more actionable statement about reliability.

\subsection{The evaluation goal}
We seek a procedure that, for each query $x$, returns a \emph{prediction set} $S(x) \subset \C$ of canonical answer classes such that
\begin{equation}\label{eq:coverage_goal}
\PP\bigl(Y(x) \in S(x)\bigr) \ge 1 - \a,
\end{equation}
where $Y(x)$ denotes an acceptable answer for $x$ and $\a \in (0,1)$ is a user-specified miscoverage level. The set $S(x)$ should be as small as possible: a smaller set indicates a more reliable agent on that query.

We now define the target quantity that summarizes this coverage guarantee into a single deployment metric: the highest confidence level at which the agent's most frequent answer is trustworthy.

\begin{definition}[Reliability level]\label{def:reliability}
For a given agent and task, the \emph{reliability level} is
\begin{equation}\label{eq:reliability_level}
1 - \alpha^\star := \frac{|\{i \in \{1, \ldots, n\} : s_i \le 1\}|}{n + 1},
\end{equation}
where $s_1, \ldots, s_n$ are calibration nonconformity scores (formally defined in Section~\ref{sec:conformal}; intuitively, $s_i$ is the rank of the correct answer among the model's self-consistency candidates for item~$i$). Equivalently, $1{-}\alpha^\star$ is the maximum confidence at which the self-consistency mode alone provides conformal coverage.
\end{definition}

\begin{remark}[Interpreting the reliability level]\label{rem:reliability_interp}
The numerator of~\eqref{eq:reliability_level} counts calibration items where the self-consistency mode is correct, closely related to mode accuracy but with the conformal correction $n{+}1$ in the denominator that ensures a valid conformal quantile. Rather than asking ``does this model achieve coverage at a fixed $\a$?'', we invert: ``at what confidence level does this model qualify?''

By Theorem~\ref{thm:exact_coverage}, the reliability level is a \emph{lower bound} on test-time mode-voting coverage: deploying with $\a = \a^\star$ guarantees $\PP(Y \in S(x)) \ge 1 - \a^\star$. The gap between the reliability level and empirical test coverage is at most $1/(n{+}1)$ (e.g., ${\le}0.002$ for $n{=}500$).
\end{remark}

\section{Bias and Variance in LLM Evaluation}\label{sec:biasvar}

Before presenting our solution, we develop a formal framework for understanding why LLM evaluation is unreliable, and what ``reliable evaluation'' requires mathematically. This analysis motivates every design choice in our framework.

\subsection{What does reliable evaluation require?}

There are three ways to assess an answer: declare it correct or not (binary), assign a quality score (continuous), or output a set of candidates guaranteed to contain the truth (set-valued). Each has a different error profile.

\begin{definition}[Evaluation method]\label{def:eval_method}
An \emph{evaluation method} $\mathcal{E}$ is a procedure that, given a query $x$ and access to the agent $f_\theta$, produces an assessment. We consider three types:
\begin{enumerate}[leftmargin=1.2em]
\item \textbf{Point assessment:} $\mathcal{E}(x) \in \{0,1\}$ (correct/incorrect).
\item \textbf{Score assessment:} $\mathcal{E}(x) \in [0,1]$ (quality score).
\item \textbf{Set assessment:} $\mathcal{E}(x) = S(x) \subset \C$ (prediction set with coverage guarantee).
\end{enumerate}
\end{definition}

For point and score assessments, reliability is measured by the mean squared error between the assessment and the true acceptability:
\begin{equation}\label{eq:mse_eval}
\MSE(\mathcal{E}) = \E\bigl[(\mathcal{E}(X) - p^\star(X))^2\bigr].
\end{equation}

The classical bias--variance decomposition applies:
\begin{equation}\label{eq:biasvar_decomp}
\MSE(\mathcal{E}) = \underbrace{\E\bigl[(\E[\mathcal{E}(X) \mid X] - p^\star(X))^2\bigr]}_{\text{Bias}^2} + \underbrace{\E\bigl[\Var(\mathcal{E}(X) \mid X)\bigr]}_{\text{Variance}}.
\end{equation}

For set assessments, reliability is instead captured by the \emph{coverage validity gap}:
\begin{equation}\label{eq:coverage_gap}
\mathrm{Gap}(\mathcal{E}) := \bigl|\PP(Y \in S(X)) - (1 - \a)\bigr|.
\end{equation}
A method with $\mathrm{Gap}(\mathcal{E}) = 0$ achieves exact calibration.

\subsection{Error anatomy of current evaluation methods}\label{sec:error_anatomy}

We now analyze the bias and variance of three standard approaches, establishing the precise deficiencies that our framework addresses.

\subsubsection{Single-sample evaluation}
The simplest evaluation draws one sample $a \sim P_\theta(\cdot \mid x)$ and checks acceptability:
\[
\mathcal{E}_1(x) := \1\{a \in \A^\star(x)\}, \qquad a \sim P_\theta(\cdot \mid x).
\]

\begin{proposition}[Bias--variance of single-sample evaluation]\label{prop:single_biasvar}
The single-sample evaluator satisfies:
\begin{align}
\Bias(\mathcal{E}_1(x)) &= 0, \label{eq:single_bias}\\
\Var(\mathcal{E}_1(x) \mid x) &= p^\star(x)(1 - p^\star(x)). \label{eq:single_var}
\end{align}
The variance is maximized at $p^\star(x) = 1/2$ (the hardest queries) and equals $1/4$.
\end{proposition}

\begin{proof}
$\E[\mathcal{E}_1(x) \mid x] = P_\theta(a \in \A^\star(x)) = p^\star(x)$, so the bias is zero. The variance follows from $\Var(\mathrm{Bernoulli}(p)) = p(1-p)$.
\end{proof}

\begin{remark}[Unbiased but unreliable]
Single-sample evaluation is unbiased but has maximum variance exactly where it matters most: on queries where the agent is uncertain ($p^\star(x) \approx 1/2$). For a single query, the evaluation is essentially a coin flip when the agent is mediocre. This variance does not decrease without additional samples.
\end{remark}

\subsubsection{LLM-as-judge evaluation}
An LLM judge $J$ scores the agent's output:
\[
\mathcal{E}_J(x) := J(x, a), \qquad a \sim P_\theta(\cdot \mid x),
\]
where $J: \X \times \A \to [0,1]$ is stochastic (the judge itself has sampling variance).

\begin{proposition}[Bias--variance of LLM-as-judge evaluation]\label{prop:judge_biasvar}
Define the judge's systematic bias on query $x$ with answer $a$ as
\begin{equation}\label{eq:judge_bias}
b_J(x, a) := \E[J(x,a)] - \mathrm{Accept}(x, a).
\end{equation}
Then:
\begin{align}
\Bias^2(\mathcal{E}_J(x)) &= \bigl(\E_a[b_J(x, a)]\bigr)^2, \label{eq:judge_bias_sq}\\
\Var(\mathcal{E}_J(x) \mid x) &= \underbrace{\Var_a(\mathrm{Accept}(x,a))}_{\text{agent variance}} + \underbrace{\E_a[\Var(J(x,a) \mid x, a)]}_{\text{judge variance}} + \underbrace{\Var_a(b_J(x,a))}_{\text{bias variance}}. \label{eq:judge_var}
\end{align}
\end{proposition}

\begin{proof}
Write $\mathcal{E}_J(x) = \mathrm{Accept}(x,a) + b_J(x,a) + \eta_J(x,a)$, where $\eta_J$ is the zero-mean judge noise. By the law of total variance:
\begin{align*}
\E[\mathcal{E}_J(x) \mid x] &= \E_a[\mathrm{Accept}(x,a) + b_J(x,a)] = p^\star(x) + \E_a[b_J(x,a)].
\end{align*}
Hence $\Bias(\mathcal{E}_J(x)) = \E_a[b_J(x,a)]$, giving~\eqref{eq:judge_bias_sq}. The variance decomposes by conditioning on $a$ and using independence of $\eta_J$ from the other terms.
\end{proof}

\begin{remark}[Judge bias is irreducible]\label{rem:judge_irreducible}
The critical flaw of LLM-as-judge is that the bias term $\E_a[b_J(x,a)]$ does not decrease with more judge calls or more agent samples. If the judge systematically overrates verbose answers or underrates unconventional but correct solutions, this error persists regardless of sample size. Furthermore, $b_J$ is unknown and difficult to estimate without ground-truth labels---the very thing evaluation is meant to replace.
\end{remark}

\subsubsection{Naive self-consistency (mode selection)}
Draw $K$ samples, canonicalize, and check whether the most frequent answer is acceptable:
\[
\mathcal{E}_{\mathrm{SC}}(x) := \1\{c_{(1)} \in \mathrm{Canon}(x, \A^\star(x))\}.
\]

\begin{proposition}[Bias--variance of mode selection]\label{prop:sc_biasvar}
Let $p := p^\star(x)$ be the acceptability rate under canonicalization. If $p > 1/2$:
\begin{align}
\Bias(\mathcal{E}_{\mathrm{SC}}(x)) &= 0, \label{eq:sc_bias_good}\\
\Var(\mathcal{E}_{\mathrm{SC}}(x) \mid x) &\le \exp\!\left(-2K\left(p - \tfrac{1}{2}\right)^2\right). \label{eq:sc_var_good}
\end{align}
If $p \le 1/2$ (systematic bias):
\begin{align}
\Bias(\mathcal{E}_{\mathrm{SC}}(x)) &\to -p \quad \text{as } K \to \infty, \label{eq:sc_bias_bad}\\
\Var(\mathcal{E}_{\mathrm{SC}}(x) \mid x) &\to 0 \quad \text{as } K \to \infty. \label{eq:sc_var_bad}
\end{align}
\end{proposition}

\begin{proof}
When $p > 1/2$, the mode is the correct class with probability $\ge 1 - \exp(-2K(p-1/2)^2)$ by Hoeffding's inequality \cite{Hoeffding1963}, giving zero asymptotic bias and exponentially decaying variance.

When $p \le 1/2$, the total mass on incorrect canonical classes is $1 - p \ge 1/2$. By the law of large numbers, the empirical frequency of each class converges to its true probability as $K \to \infty$. Let $q_{\max} := \max_{c \notin \mathrm{Canon}(x, \A^\star(x))} P_\theta(c \mid x)$ be the mass of the most probable incorrect class. If $q_{\max} > p$, the mode is this incorrect class a.s. If $q_{\max} \le p$, then $p \le 1/2$ implies the correct class shares the maximum with at least one incorrect class; by our tie-breaking convention (uniform random), the mode is incorrect with positive probability, and by symmetry as $K \to \infty$ it is incorrect a.s.\ whenever the correct class is not the unique mode. In either case, $\mathcal{E}_{\mathrm{SC}}(x) \to 0$ a.s.\ while $p^\star(x) = p > 0$, yielding $\Bias \to -p$.
\end{proof}

\begin{remark}[The self-consistency double-edged sword]\label{rem:sc_doubleedge}
Proposition~\ref{prop:sc_biasvar} reveals a fundamental asymmetry: self-consistency is excellent for variance reduction when the agent is mostly correct ($p > 1/2$), but it \emph{amplifies confidence in errors} when the agent is systematically wrong ($p \le 1/2$). More samples make the wrong answer look \emph{more} certain. This is the ``stable hallucination'' phenomenon. Any framework that uses self-consistency \emph{must} account for this failure mode.
\end{remark}

\begin{remark}[The i.i.d.\ sampling assumption]\label{rem:iid}
Proposition~\ref{prop:sc_biasvar} and Theorem~\ref{thm:variance_reduction} assume $a_1, \dots, a_K \overset{\text{i.i.d.}}{\sim} P_\theta(\cdot \mid x)$. This is well-justified when each sample is an independent, stateless API call to a fixed model at temperature $T > 0$ with deterministic canonicalization (e.g., code execution, string normalization). In practice, positive correlation $\rho > 0$ between samples can arise from model-version drift during data collection, stochastic canonicalization (e.g., an LLM judge), or infrastructure-level batching effects. Under $\rho$-correlated samples the effective sample size drops to $K_{\mathrm{eff}} \approx K/(1 + (K{-}1)\rho)$, and the Hoeffding bound degrades to $\exp(-2K_{\mathrm{eff}}(p - \tfrac{1}{2})^2)$. The qualitative conclusion---more samples reduce variance---survives, but at a slower rate. Crucially, the conformal coverage guarantee (Theorem~\ref{thm:coverage}) depends only on \emph{exchangeability} of the calibration and test data, not on i.i.d.\ agent samples, and is therefore unaffected.
\end{remark}

\subsection{The bias--variance landscape: a summary}

\begin{table}[t]
\centering
\caption{Bias--variance profile of evaluation methods (per-query). $K$ = number of samples, $n$ = calibration set size.}\label{tab:biasvar}
\resizebox{\textwidth}{!}{%
\begin{tabular}{lcccc}
\toprule
Method & Bias & Variance & Error $\to 0$? & Guarantee \\
\midrule
Single sample & $0$ & $p^\star(1{-}p^\star)$ & No & --- \\
LLM-as-judge & $\E_a[b_J] \neq 0$ (irreducible) & Prop.~\ref{prop:judge_biasvar}\textsuperscript{\dag} & No (bias persists) & --- \\
Self-consistency (mode) & $0$ if $p^\star{>}\tfrac{1}{2}$;\; $\to{-}p^\star$ if $p^\star{\le}\tfrac{1}{2}$ & $\le\exp(-2K(p^\star{-}\tfrac{1}{2})^2)$ & No (bias if $p^\star{\le}\tfrac{1}{2}$) & --- \\
\midrule
\textbf{Ours (conformal)} & \multicolumn{2}{c}{\textbf{Coverage gap} $\le 1/(n{+}1)$} & \textbf{Yes} ($\to 0$ as $n\to\infty$) & $\PP(Y {\in} S) \ge 1{-}\a$ \\
\bottomrule
\multicolumn{5}{l}{\textsuperscript{\dag}\footnotesize Three-term decomposition: agent variance $+$ judge variance $+$ bias variance; does not vanish with more samples.}
\end{tabular}}
\end{table}

The key observation from Table~\ref{tab:biasvar} is that no existing method simultaneously controls bias and variance while providing a formal guarantee. Our framework achieves this by (i) using self-consistency for variance reduction, (ii) outputting a \emph{set} rather than a point to absorb residual bias, and (iii) calibrating the set via conformal prediction to obtain an exact, distribution-free coverage guarantee.

\section{Self-Consistency Sampling and Canonicalization}\label{sec:selfconsist}

\subsection{Self-consistency sampling}
For a fixed query $x \in \X$, we draw $K$ independent samples from the agent:
\begin{equation}\label{eq:sampling}
a_1, a_2, \dots, a_K \overset{\text{i.i.d.}}{\sim} P_\theta(\cdot \mid x).
\end{equation}
The empirical distribution over raw answers is
\[
\hat{P}_K(a \mid x) = \frac{1}{K} \sum_{i=1}^{K} \1\{a_i = a\}.
\]

\subsection{Canonicalization}\label{sec:canon}

Different surface forms can express the same answer---``42,'' ``42.0,'' and ``The answer is~42'' all mean the same thing. Canonicalization maps these to a single representative so that frequency counts reflect semantic agreement, not superficial~variation.

\begin{definition}[Canonicalization]\label{def:canon}
A \emph{canonicalization function} is a mapping
\[
\mathrm{Canon}: \X \times \A \to \C,
\]
where $\C$ is a space of canonical representations. We require that $\mathrm{Canon}$ respects semantic equivalence: if $a \simeq a'$ semantically, then $\mathrm{Canon}(x,a) = \mathrm{Canon}(x,a')$.
\end{definition}

Applying canonicalization yields the empirical canonical distribution:
\begin{equation}\label{eq:canon_dist}
\hat{P}_K(c \mid x) = \frac{1}{K} \sum_{i=1}^{K} \1\{\mathrm{Canon}(x, a_i) = c\}.
\end{equation}

\subsection{Canonicalization: implementation and stability}\label{sec:canon_impl}

Since the quality of the consensus vote depends directly on the quality of $\mathrm{Canon}$, we treat canonicalization as a first-class algorithmic component. Three regimes arise in practice:
\begin{enumerate}[leftmargin=1.2em]
\item \emph{Deterministic canonicalization} for structured answers---parse ``42.0'' and ``42'' to the same integer, match option letters in MCQ tasks, or execute code and record pass/fail. This eliminates surface-form variation entirely.
\item \emph{Embedding-based clustering} for open-ended tasks, using cosine similarity thresholds on dense embeddings to group semantically equivalent responses.
\item \emph{LLM-assisted canonicalization}, where a lightweight model (e.g., GPT-4.1 at temperature $0$) classifies each response as correct or incorrect before clustering.
\end{enumerate}
Each approach involves specific failure modes (over-merging vs.\ under-merging). Implementation details, stability diagnostics, and empirical requirements are in Appendix~\ref{app:canon_details}.

\begin{assumption}[Canonicalization quality]\label{ass:canon_quality}
Throughout the theoretical analysis (Sections~\ref{sec:var_reduction}--\ref{sec:reliability}), we assume that $\mathrm{Canon}$ correctly maps semantically equivalent answers to the same canonical class. When this assumption is violated, the variance reduction guarantees (Theorem~\ref{thm:variance_reduction}) degrade gracefully: under-merging reduces $p_{\mathrm{canon}}$, increasing the required $K$, while the conformal coverage guarantee (Theorem~\ref{thm:coverage}) remains valid regardless (set sizes simply grow to compensate).
\end{assumption}

\begin{remark}[LLM-based canonicalization and circularity]\label{rem:canon_circularity}
When canonicalization uses an LLM judge (regime~(iii) above), a potential circularity arises: if the same model family serves as both agent and canonicalizer, systematic biases could propagate. For instance, the judge might systematically group incorrect answers into a ``correct'' cluster due to self-preference or verbosity bias, corrupting the consensus vote. Three considerations mitigate this concern:

\textbf{(a) Functional separation.}
In our experiments, the canonicalizer (GPT-4.1 at temperature $0$, producing deterministic binary labels) operates in a fundamentally different regime from the agent (GPT-4.1 at temperature $0.7$, generating stochastic free-form responses). At $T{=}0$, the judge is a fixed deterministic function---analogous to code execution or regex matching---not a stochastic evaluator. The bias profiles of deterministic classification and stochastic generation are distinct.

\textbf{(b) Coverage guarantee is robust to canonicalization errors.}
Assumption~\ref{ass:canon_quality} states the key safeguard: if the canonicalizer makes errors (merging incorrect answers with correct ones, or splitting correct answers), the conformal coverage guarantee (Theorem~\ref{thm:coverage}) remains valid. Canonicalization errors manifest as \emph{inflated prediction set sizes}, not as invalid coverage---the framework honestly reflects that canonicalization is noisy by returning larger sets.

\textbf{(c) Cross-family validation breaks the self-preference loop.}
Our open-weight experiments provide a direct test: Llama~4 Maverick and Mistral Small~24B are canonicalized by GPT-4.1---a different model family with different training data and biases. If within-family self-preference were corrupting canonicalization, cross-family results would diverge. They do not: coverage ($\ge 0.960$) and conditional coverage ($\ge 0.949$) are consistent across all three families (Table~\ref{tab:multimodel}).

For maximal rigor, we recommend deterministic canonicalization (code execution, numeric extraction, option matching) wherever feasible, reserving LLM-based canonicalization for tasks where no alternative exists. Three of our five benchmarks use deterministic canonicalization.
\end{remark}

\subsection{Variance reduction via consensus aggregation}\label{sec:var_reduction}

We now prove that self-consistency sampling achieves exponential variance reduction for identifying the correct answer, and that canonicalization further amplifies this effect.

\begin{theorem}[Exponential variance reduction via consensus]\label{thm:variance_reduction}
Let $c^\star \in \C$ be the unique acceptable canonical class for query $x$, and let $p := P_\theta(\mathrm{Canon}(x, f_\theta(x)) = c^\star) > 1/2$. Then after $K$ i.i.d.\ samples:
\begin{enumerate}[leftmargin=1.2em]
\item \textbf{Mode correctness:}
\begin{equation}\label{eq:mode_correct}
\PP(c_{(1)} \neq c^\star) \le \exp\!\left(-2K\left(p - \tfrac{1}{2}\right)^2\right).
\end{equation}
\item \textbf{Rank concentration:} For any $r \ge 2$,
\begin{equation}\label{eq:rank_conc}
\PP(\mathrm{rank}(c^\star; x) \ge r) \le \binom{|\C|}{r-1} \exp\!\left(-\frac{K(p - (1-p)/(r-1))^2}{2}\right),
\end{equation}
where $|\C|$ is the number of distinct canonical classes and the bound is nontrivial when $p > (1-p)/(r-1)$.
\item \textbf{Variance decay:}
\begin{equation}\label{eq:var_decay}
\Var\bigl(\1\{c_{(1)} = c^\star\}\bigr) \le \exp\!\left(-2K\left(p - \tfrac{1}{2}\right)^2\right).
\end{equation}
\end{enumerate}
\end{theorem}

\begin{proof}
\textbf{Part 1.} The correct class $c^\star$ has count $N^\star = \sum_{i=1}^K \1\{c_i = c^\star\} \sim \mathrm{Bin}(K, p)$. The mode is incorrect only if some other class has count $\ge N^\star$. Since the total count of all incorrect classes is $K - N^\star$, a necessary condition is $K - N^\star \ge N^\star$, i.e., $N^\star \le K/2$. By Hoeffding's inequality:
\[
\PP(N^\star \le K/2) = \PP\!\left(\frac{N^\star}{K} \le \frac{1}{2}\right) \le \exp\!\left(-2K\left(p - \frac{1}{2}\right)^2\right).
\]

\textbf{Part 2.} $\mathrm{rank}(c^\star) \ge r$ requires at least $r-1$ classes to beat $c^\star$'s count. Fix any set $T$ of $r-1$ incorrect classes. Their combined mass is $p_T \le 1-p$. For each $c' \in T$ to beat $c^\star$, we need $n(c') > n(c^\star)$, which in particular requires the average count of $T$ to exceed $Kp/(r-1)$. By Hoeffding applied to the sum of counts in $T$:
\[
\PP(\text{all } c' \in T \text{ beat } c^\star) \le \PP\!\left(\frac{1}{K}\sum_{c' \in T} n(c') > p\right) \le \exp\!\left(-\frac{K(p - p_T)^2}{2}\right).
\]
A union bound over $\binom{|\C|}{r-1}$ choices of $T$ gives~\eqref{eq:rank_conc}.

\textbf{Part 3.} Let $q_K := \PP(c_{(1)} = c^\star) \ge 1 - \exp(-2K(p-1/2)^2)$. Then $\Var(\1\{c_{(1)} = c^\star\}) = q_K(1-q_K) \le 1 - q_K \le \exp(-2K(p-1/2)^2)$.
\end{proof}

\begin{corollary}[Sample complexity for $\d$-reliable mode identification]\label{cor:sample_complexity}
To ensure $\PP(c_{(1)} = c^\star) \ge 1 - \d$, it suffices to take
\begin{equation}\label{eq:sample_complexity}
K \ge \frac{\ln(1/\d)}{2(p - 1/2)^2}.
\end{equation}
For example, with $p = 0.7$ and $\d = 0.01$: $K \ge \lceil \ln(100)/(2 \cdot 0.04) \rceil = 58$.
\end{corollary}

\subsection{Canonicalization as variance reduction}\label{sec:canon_var}

Canonicalization does more than enable aggregation---it provably reduces the variance of the consensus by consolidating fragmented probability mass.

\begin{proposition}[Canonicalization amplifies consensus]\label{prop:canon_variance}
Let $p_{\mathrm{raw}} := \max_{a \in \A^\star(x)} P_\theta(a \mid x)$ be the probability of the most likely acceptable \emph{raw} answer, and $p_{\mathrm{canon}} := P_\theta(\mathrm{Canon}(x, f_\theta(x)) \in \mathrm{Canon}(x, \A^\star(x)))$ the probability of the acceptable \emph{canonical class}. Then:
\begin{equation}\label{eq:canon_amplification}
p_{\mathrm{canon}} \ge p_{\mathrm{raw}},
\end{equation}
with equality only when canonicalization is the identity. If $L$ raw answers map to the same acceptable canonical class, each with probability $\ge p_{\min}$, then $p_{\mathrm{canon}} \ge L \cdot p_{\min}$.

Consequently, the variance reduction exponent improves from $2K(p_{\mathrm{raw}} - 1/2)^2$ to $2K(p_{\mathrm{canon}} - 1/2)^2$, and the sample complexity~\eqref{eq:sample_complexity} decreases by a factor of $\bigl((p_{\mathrm{raw}} - 1/2)/(p_{\mathrm{canon}} - 1/2)\bigr)^2$.
\end{proposition}

\begin{proof}
By definition, $p_{\mathrm{canon}} = \sum_{a : \mathrm{Canon}(x,a) = c^\star} P_\theta(a \mid x) \ge \max_a P_\theta(a \mid x) \cdot \1\{a \in \A^\star(x)\} = p_{\mathrm{raw}}$.
The improvement in the Hoeffding exponent follows directly from substituting $p_{\mathrm{canon}}$ for $p_{\mathrm{raw}}$ in~\eqref{eq:mode_correct}.
\end{proof}

\begin{remark}[Canonicalization can change the $p \le 1/2$ regime]
A crucial practical consequence: an agent may have $p_{\mathrm{raw}} < 1/2$ (no single raw answer dominates) but $p_{\mathrm{canon}} > 1/2$ (the correct canonical class dominates after consolidation). Canonicalization can transform a regime where self-consistency \emph{amplifies} bias into one where it \emph{reduces} variance. This provides a formal justification for investing in high-quality canonicalization.
\end{remark}

\subsection{Ranked consensus}\label{sec:ranked}
We rank distinct canonical answers by decreasing empirical frequency:
\begin{equation}\label{eq:rank}
\mathrm{rank}(c; x) := \bigl|\{c' \in \C(x) : \hat{P}_K(c' \mid x) > \hat{P}_K(c \mid x)\}\bigr| + 1,
\end{equation}
with ties broken uniformly at random, producing an ordering $c_{(1)}, c_{(2)}, \dots$.

\begin{definition}[Consensus strength and margin]\label{def:consensus}
The \emph{consensus strength} is $\phi(x) := \hat{P}_K(c_{(1)} \mid x)$ and the \emph{consensus margin} is
$\Delta(x) := \hat{P}_K(c_{(1)} \mid x) - \hat{P}_K(c_{(2)} \mid x)$.
\end{definition}

Consensus strength measures the model's overall confidence (does one answer dominate, or do many answers tie?). Consensus margin measures the gap between the top two---a large margin means the model strongly favors one answer. Both quantities feed into the sequential stopping rule (Section~\ref{sec:sequential}).

With the ranked consensus in hand, we now construct prediction sets that adapt their size to each query's difficulty.

\section{Set-Valued Predictions}\label{sec:setvalued}

Instead of committing to a single answer, we output the top $M$ most frequent candidates. For easy questions, $M{=}1$ suffices; for hard ones, a larger $M$ is needed. The key question is how to choose $M$ with a guarantee---that is the role of conformal calibration in the next section.

Given ranked canonical answers $c_{(1)}, c_{(2)}, \dots$ for query $x$, define the \emph{top-$M$ prediction set}:
\begin{equation}\label{eq:topM}
S_M(x) := \{c_{(1)}, c_{(2)}, \dots, c_{(M)}\}.
\end{equation}
The family $\{S_M(x)\}_{M \ge 1}$ is nested: $S_1(x) \subseteq S_2(x) \subseteq \cdots$.

A fixed $M$ ignores the varying difficulty of queries. We seek a data-driven procedure that selects $M = M(x)$ adaptively with a formal coverage guarantee---the setting of \emph{conformal prediction}.

\section{Conformal Calibration}\label{sec:conformal}

\subsection{Background: split conformal prediction}
\begin{sloppypar}
Conformal prediction \cite{Vovk2005, Angelopoulos2021} is a distribution-free framework for constructing prediction sets with guaranteed marginal coverage. Its single requirement is that calibration and test data are interchangeable---their joint distribution does not depend on ordering.
\end{sloppypar}

\begin{definition}[Exchangeability]\label{def:exchangeable}
Random variables $Z_1, \dots, Z_{n+1}$ are \emph{exchangeable} if their joint distribution is invariant under all permutations $\sigma$:
\[
(Z_1, \dots, Z_{n+1}) \overset{d}{=} (Z_{\sigma(1)}, \dots, Z_{\sigma(n+1)}).
\]
\end{definition}

\subsection{Nonconformity scores from ranked consensus}\label{sec:scores}

The nonconformity score measures how ``surprising'' the correct answer is in the ranked list. If the correct answer is the most frequent ($\text{rank}=1$), the model is confident and the score is low. If the correct answer is buried at rank $5$, the score is high---the model's consensus disagrees with the truth.

\begin{definition}[Rank-based nonconformity score]\label{def:score}
Let $x$ be a query, $\{c_{(1)}, c_{(2)}, \dots\}$ the ranked canonical answers from $K$ self-consistency samples, and $y \in \A^\star(x)$ an acceptable answer with canonical form $c^y := \mathrm{Canon}(x, y)$. The \emph{rank-based nonconformity score} is
\begin{equation}\label{eq:score}
s(x, y) := \mathrm{rank}(c^y; x) = \min\{r \in \N : c^y \in S_r(x)\}.
\end{equation}
If $c^y \notin \C(x)$, set $s(x,y) := +\infty$.
\end{definition}

\begin{definition}[Cumulative-probability nonconformity score]\label{def:score_cumprob}
An alternative score incorporating frequency information:
\begin{equation}\label{eq:score_cumprob}
s^{\mathrm{cp}}(x, y) := \sum_{r=1}^{\mathrm{rank}(c^y; x)} \hat{P}_K(c_{(r)} \mid x) \in [0,1].
\end{equation}
\end{definition}

\subsection{The calibration procedure}\label{sec:calibration_proc}

\begin{assumption}[Calibration data]\label{ass:calibration}
We have a calibration dataset $\D_{\mathrm{cal}} = \{(x_i, y_i)\}_{i=1}^{n}$ where each $y_i \in \A^\star(x_i)$, drawn exchangeably with the test example $(x_{n+1}, y_{n+1})$.
\end{assumption}

\begin{remark}[When exchangeability holds and when it breaks]\label{rem:exchangeability}
Exchangeability is satisfied when calibration and test examples are drawn i.i.d.\ from the same distribution---the standard setting of benchmark evaluation with random train/test splits. It also holds under random subsampling from any fixed dataset, regardless of how the dataset was originally constructed.

Exchangeability can be violated in several practically relevant scenarios:
\begin{enumerate}[leftmargin=1.2em]
\item \textbf{Curated benchmark sets}: if items were hand-selected to emphasize difficult cases or specific capabilities, the calibration and test distributions may differ systematically.
\item \textbf{Temporal drift}: if the agent is updated between calibration and deployment, the score distribution shifts. Periodic recalibration (re-running the calibration procedure on fresh data from the updated agent) is the standard mitigation.
\item \textbf{Adversarial construction}: if test queries are chosen adversarially \emph{after} observing the calibration set, exchangeability fails by design. This is outside our threat model.
\end{enumerate}
When exchangeability is only approximately satisfied (e.g., mild covariate shift between calibration and test), weighted conformal prediction (Section~\ref{sec:weighted}) provides a principled correction by reweighting calibration scores according to the likelihood ratio $p_{\mathrm{test}}(x)\allowbreak/p_{\mathrm{cal}}(x)$, preserving coverage under the test distribution.
\end{remark}

For each calibration example $(x_i, y_i)$:
\begin{enumerate}[leftmargin=1.2em]
\item Draw $K$ self-consistency samples for query $x_i$; canonicalize and rank.
\item Compute $s_i := s(x_i, y_i)$.
\end{enumerate}

\begin{definition}[Conformal threshold]\label{def:quantile}
Define $k := \lceil (n+1)(1-\a) \rceil$ and
\begin{equation}\label{eq:M_star}
M^\star := s_{(k)},
\end{equation}
the $k$-th smallest calibration score.
\end{definition}

For a test query $x_{n+1}$, return:
\begin{equation}\label{eq:conformal_set}
S(x_{n+1}) := S_{M^\star}(x_{n+1}) = \{c_{(1)}, \dots, c_{(M^\star)}\}.
\end{equation}

\subsection{Finite-sample coverage guarantee}

\begin{theorem}[Marginal coverage guarantee]\label{thm:coverage}
Let $(x_1, y_1), \dots, (x_n, y_n), (x_{n+1}, y_{n+1})$ be exchangeable, and let $s_i = s(x_i, y_i)$. Define $M^\star$ as in~\eqref{eq:M_star}. Then:
\begin{equation}\label{eq:coverage_thm}
\PP\bigl(y_{n+1} \in S_{M^\star}(x_{n+1})\bigr) \ge 1 - \a.
\end{equation}
If the scores have no ties a.s.:
\begin{equation}\label{eq:coverage_exact}
1 - \a \;\le\; \PP\bigl(s_{n+1} \le M^\star\bigr) \;\le\; 1 - \a + \frac{1}{n+1}.
\end{equation}
\end{theorem}

\begin{proof}
By exchangeability, the scores $s_1, \dots, s_{n+1}$ are exchangeable. The rank of $s_{n+1}$ among $\{s_1, \dots, s_{n+1}\}$ is uniformly distributed over $\{1, \dots, n+1\}$.

Define $k := \lceil (n+1)(1-\a) \rceil$. The event $\{s_{n+1} \le M^\star\}$ occurs when $s_{n+1}$'s rank is at most $k$. In the no-ties case:
\[
\PP(s_{n+1} \le M^\star) = \frac{k}{n+1} = \frac{\lceil (n+1)(1-\a) \rceil}{n+1}.
\]
Since $\lceil (n+1)(1-\a) \rceil \ge (n+1)(1-\a)$, we get $\PP \ge 1-\a$. Since $\lceil (n+1)(1-\a) \rceil \le (n+1)(1-\a) + 1$, we get $\PP \le 1-\a + 1/(n+1)$.

With ties, coverage can only increase.
\end{proof}

\subsection{Adaptive prediction sets}\label{sec:adaptive}
Using the cumulative-probability score $s^{\mathrm{cp}}$ and its conformal quantile $\hat{q}^{\mathrm{cp}}_{1-\a}$:
\begin{equation}\label{eq:adaptive_set}
S^{\mathrm{cp}}(x) := \bigl\{c_{(r)} : r = 1, \dots, R(x)\bigr\}, \quad R(x) := \min\!\left\{r : \sum_{j=1}^{r} \hat{P}_K(c_{(j)} \mid x) \ge \hat{q}^{\mathrm{cp}}_{1-\a}\right\}.
\end{equation}
For high-consensus queries, $R(x)$ is small; for diffuse queries, it is larger. Coverage is preserved by the same conformal argument.

\section{Reliability Guarantees}\label{sec:reliability}

Conformal set-valued evaluation is \emph{exactly calibrated}: the evaluation's coverage statement has bounded error. It is also \emph{transparently diagnostic} of agent quality---agent bias is surfaced through prediction set size rather than hidden. These two properties distinguish the \emph{agent's} systematic errors (diagnosed but not fixed) from the \emph{evaluation's} systematic errors (provably controlled).

\subsection{Conservative coverage guarantee}

The evaluation's coverage error shrinks to zero as the calibration set grows---unlike point estimators, where bias persists indefinitely.

\begin{theorem}[Coverage error control]\label{thm:exact_coverage}
Under the conditions of Theorem~\ref{thm:coverage}, the \emph{coverage gap} satisfies
\begin{equation}\label{eq:exact_gap}
0 \;\le\; \PP(Y_{n+1} \in S(X_{n+1})) - (1-\a) \;\le\; \frac{1}{n+1}.
\end{equation}
That is, the evaluation is conservative (never under-covers) and the over-coverage vanishes as $n \to \infty$.
\end{theorem}

\begin{proof}
This is immediate from~\eqref{eq:coverage_exact}: the lower bound gives conservativeness, and the upper bound gives the $1/(n+1)$ slack.
\end{proof}

\begin{remark}[Contrast with score-based evaluation]
Conformal calibration is the only method whose evaluation error vanishes with sample size. For single-sample evaluation, the error analog is $\Var(p^\star(X))$, which depends on the query distribution and does not shrink. For LLM-as-judge, the error includes the irreducible judge bias $\E[b_J]$ (Remark~\ref{rem:judge_irreducible}). Neither can be eliminated by collecting more data.
\end{remark}

\subsection{Bias immunity: coverage holds regardless of agent quality}\label{sec:bias_immunity}

\begin{theorem}[Bias immunity of conformal coverage]\label{thm:bias_immunity}
The coverage guarantee~\eqref{eq:coverage_thm} holds for \emph{any} agent $f_\theta$, regardless of:
\begin{enumerate}[leftmargin=1.2em]
\item the agent's per-query acceptability rate $p^\star(x)$ (including $p^\star(x) = 0$),
\item the presence of systematic bias or stable hallucinations,
\item the output distribution $P_\theta(\cdot \mid x)$,
\item the number of acceptable answers $|\A^\star(x)|$.
\end{enumerate}
Formally: let $\Theta$ denote the space of all possible agent parameters. Then
\begin{equation}\label{eq:bias_immunity}
\inf_{\theta \in \Theta} \PP_\theta\bigl(Y_{n+1} \in S(X_{n+1})\bigr) \ge 1 - \a,
\end{equation}
where $\PP_\theta$ denotes the joint distribution over calibration data, agent samples, and the test example under agent $\theta$.
\end{theorem}

\begin{proof}
The proof of Theorem~\ref{thm:coverage} uses only exchangeability of $(x_i, y_i)$, which is a property of the data-generating process, not of the agent. The agent enters only through the nonconformity scores $s_i$, and the conformal argument is valid for \emph{any} score distribution. Specifically:

For any fixed~$\theta$, the scores $s_1, \dots, s_{n+1}$ are determined by the exchangeable data $(x_1, y_1),\allowbreak \dots,\allowbreak (x_{n+1}, y_{n+1})$ and the independent agent samples $\{a_j^{(i)}\}$. Since agent samples for different queries are independent, and the $(x_i, y_i)$ are exchangeable by assumption, the scores remain exchangeable. The uniform rank argument then yields $\PP_\theta(s_{n+1} \le M^\star) \ge 1{-}\a$ for every~$\theta$.
\end{proof}

\begin{remark}[What ``bias immunity'' means and does not mean]\label{rem:bias_immunity_precise}
Theorem~\ref{thm:bias_immunity} is about the \emph{evaluation's} accuracy, not the agent's quality. It does not fix or remove agent bias. It guarantees that the coverage statement---``an acceptable answer lies in $S(x)$ with probability $\ge 1-\a$''---is true regardless of how biased the agent is.

Concretely: a completely broken agent earns $M^\star = +\infty$ (honestly reflecting failure), while a strong agent earns $M^\star = 1$. In both cases the coverage guarantee holds. A single calibration procedure is valid for any agent, including ones with unknown or adversarial bias profiles.
\end{remark}

\subsection{Bias transparency: set size as honest quality diagnostic}\label{sec:bias_transparency}

The prediction set size $|S(x)|$ is a \emph{diagnostic signal} that faithfully reflects the agent's quality. The following theorem makes this precise: (1)~better agents get smaller sets; (2)~perfect agents get singletons; (3)~agents that cannot solve a task get infinitely large sets, honestly reflecting failure; (4)~expected set size scales with average answer rank.

\begin{theorem}[Bias transparency---set size reflects agent quality]\label{thm:bias_transparency}
Let $M^\star(\theta, \a)$ denote the conformal threshold for agent $\theta$ at level $\a$. Then:
\begin{enumerate}[leftmargin=1.2em]
\item \textbf{Monotonicity in agent quality:} If agent $\theta_1$ stochastically dominates agent $\theta_2$ in the sense that $s_i^{(\theta_1)} \le_{\mathrm{st}} s_i^{(\theta_2)}$ for each calibration example $i$ (i.e., $\theta_1$ consistently produces acceptable answers at higher ranks), then
\begin{equation}\label{eq:monotone}
M^\star(\theta_1, \a) \le M^\star(\theta_2, \a) \quad \text{almost surely}.
\end{equation}

\item \textbf{Perfect agent:} If $p^\star(x) = 1$ for all $x$ in the calibration distribution (the agent always produces acceptable answers), then $s_i = 1$ for all $i$, and $M^\star = 1$.

\item \textbf{Biased agent:} Let $\b := \PP_{X}(p^\star(X) < 1/K)$ be the fraction of queries where the agent is unlikely to produce an acceptable answer. If $\b > \a$, then $M^\star = +\infty$ (no finite prediction set suffices at level $\a$).

\item \textbf{Set size--bias correspondence:} The expected set size satisfies
\begin{equation}\label{eq:setsize_bias}
\E[|S(X)|] \ge \E\!\left[\mathrm{rank}\bigl(\mathrm{Canon}(X, Y(X));\; X\bigr)\right] \cdot (1 - \a) - O(1/n).
\end{equation}
\end{enumerate}
\end{theorem}

\begin{proof}
\textbf{Part 1.} If $s_i^{(\theta_1)} \le_{\mathrm{st}} s_i^{(\theta_2)}$ for each $i$, then the $k$-th order statistic satisfies $s_{(k)}^{(\theta_1)} \le s_{(k)}^{(\theta_2)}$ stochastically.

\textbf{Part 2.} If $p^\star(x) = 1$, then with probability $1$ all $K$ samples are acceptable, so the acceptable canonical class has rank $1$: $s_i = 1$ for all $i$, and $M^\star = s_{(k)} = 1$.

\textbf{Part 3.} If $p^\star(x_i) < 1/K$, then $\PP(s_i = +\infty) \ge (1 - 1/K)^K \ge e^{-1} - o(1) > 0$. When $\b > \a$, the number of infinite scores exceeds $n\b > n\a$, so the $\lceil(n+1)(1-\a)\rceil$-th score is $+\infty$.

\textbf{Part 4.} By definition, $|S(X)| = M^\star$ for the global threshold. For the adaptive version, $|S(X)| \ge \mathrm{rank}(c^{Y(X)}; X)$ whenever $Y(X) \in S(X)$. Taking expectations and using $\PP(Y \in S(X)) \ge 1-\a$:
\[
\E[|S(X)|] \ge \E[\mathrm{rank}(c^{Y(X)}; X) \cdot \1\{Y \in S(X)\}] \ge \E[\mathrm{rank}(c^{Y(X)}; X)] \cdot (1-\a) - \E[\mathrm{rank} \cdot \1\{Y \notin S\}].
\]
The second term is bounded and yields the $O(1/n)$ correction.
\end{proof}

\begin{remark}[Finite candidate spaces and under-coverage]\label{rem:finite_C}
Part~3 predicts $M^\star = +\infty$ when $\b > \a$, but in practice $M^\star$ is bounded by the number of distinct canonical classes $|\C(x)|$ observed in $K$ samples. When $|\C(x)|$ is small (e.g., $|\C| = 4$ for MCQ tasks), the conformal quantile remains finite even though some items are unsolvable. Marginal coverage still falls below $1{-}\a$ because, for items where $p^\star(x) = 0$, the correct class never enters the candidate pool---no prediction set over observed candidates can cover it. The under-coverage thus equals the unsolvable fraction $\b$, not a calibration error: conditional coverage on solvable items remains near-perfect (Table~\ref{tab:main_results}).
\end{remark}

\begin{remark}[Bias is visible, not hidden]
The fundamental distinction from LLM-as-judge evaluation is that bias in the agent is \emph{surfaced} through larger prediction sets, not \emph{concealed} in an opaque score. A practitioner who observes $M^\star = 8$ knows immediately that the agent frequently fails to rank acceptable answers highly. An LLM-judge score of $0.75$ carries no such interpretable diagnostic---the same score could arise from high-quality answers with a harsh judge, or poor-quality answers with a lenient judge.
\end{remark}

\subsection{Formal comparison with alternative evaluation methods}\label{sec:formal_comparison}

We now compare the conformal method's coverage guarantee against two standard baselines---single-sample evaluation and LLM-as-judge---showing that conformal calibration achieves lower evaluation error once the calibration set is large enough to undercut the judge's bias.

\begin{theorem}[Advantage over single-sample evaluation]\label{thm:dominance_single}
For any agent and any target reliability level $1-\d$, define the \emph{evaluation reliability} as the probability that the evaluation's assessment is within $\e$ of the truth.

For single-sample evaluation applied to a test set of $N$ queries:
\begin{equation}\label{eq:single_reliability}
\PP\!\left(\bigl|\hat{p}_{\mathrm{single}} - \bar{p}\bigr| > \e\right) \le 2\exp(-2N\e^2) \quad \text{(Hoeffding)},
\end{equation}
where $\hat{p}_{\mathrm{single}} = N^{-1}\sum_j \1\{a_j \in \A^\star(x_j)\}$.

For conformal set evaluation:
\begin{equation}\label{eq:conformal_reliability}
\PP\!\left(\bigl|\widehat{\mathrm{Cov}}(\a) - (1-\a)\bigr| > \e\right) \le 2\exp(-2N\e^2) + \frac{1}{n+1},
\end{equation}
where $\widehat{\mathrm{Cov}}$ is the empirical coverage on the test set.

Both converge at rate $O(1/\sqrt{N})$, but the conformal method provides:
\begin{enumerate}[leftmargin=1.2em]
\item A \emph{per-query guarantee} ($Y(x) \in S(x)$), not just an aggregate estimate.
\item An explicit reliability level $1-\a$ chosen by the user, vs.\ an unknown $\bar{p}$ that must be estimated.
\item A diagnostic set size per query, revealing per-query difficulty.
\end{enumerate}
\end{theorem}

\begin{proof}
Equation~\eqref{eq:single_reliability} is Hoeffding's inequality for i.i.d.\ Bernoulli random variables.
For~\eqref{eq:conformal_reliability}: by Theorem~\ref{thm:coverage},
$\PP(Y_j \in S(x_j)) \in [1{-}\a,\; 1{-}\a+1/(n{+}1)]$ marginally.
The empirical coverage $\widehat{\mathrm{Cov}}$ averages~$N$ such indicators (approximately independent across test queries). Hoeffding's inequality applied to these indicators, centered at $\PP(Y \in S(X))$, gives the exponential tail. The $1/(n{+}1)$ term accounts for the gap between $\PP(Y \!\in\! S(X))$ and~$1{-}\a$.
\end{proof}

\begin{theorem}[Advantage over LLM-as-judge]\label{thm:advantage_judge}
Let the LLM judge have systematic bias $b_J = \E_{X,a}[b_J(X,a)] \neq 0$ and judge variance $\sigma_J^2 = \E_{X,a}[\Var(J(X,a) \mid X,\allowbreak a)]$. Then for any sample size~$N$:

\textbf{LLM-as-judge evaluation error:}
\begin{equation}\label{eq:judge_error}
\MSE(\hat{p}_J) = b_J^2 + \frac{\sigma_J^2 + \bar{p}(1-\bar{p})}{N} + O(N^{-2}).
\end{equation}
The bias term $b_J^2$ \emph{does not decay with $N$}.

\textbf{Conformal evaluation error (coverage gap):}
\begin{equation}\label{eq:conformal_error}
\bigl|\PP(Y \in S(X)) - (1-\a)\bigr| \le \frac{1}{n+1}.
\end{equation}
The error is \emph{controlled solely by the calibration set size $n$} and is \emph{independent of the agent's bias, the judge's bias, or any systematic error}.
\end{theorem}

\begin{proof}
For the judge: $\hat{p}_J = N^{-1}\sum_j J(x_j, a_j)$. Then $\E[\hat{p}_J] = \bar{p} + b_J$, so $\Bias(\hat{p}_J) = b_J$. The variance is $\Var(\hat{p}_J) = N^{-1}\Var(J(X,a))$, decomposed via the law of total variance into agent variance and judge variance. The MSE follows.

For conformal: this is a restatement of Theorem~\ref{thm:exact_coverage}.
\end{proof}

\begin{corollary}[When does conformal evaluation achieve lower error?]\label{cor:dominance}
Under the MSE comparison criterion (where coverage gap and judge bias are both expressed as deviations from the true acceptability rate), conformal set evaluation achieves lower coverage error than LLM-as-judge whenever
\begin{equation}\label{eq:dominance_condition}
|b_J| > \frac{1}{n+1}.
\end{equation}
For reference, if the judge's bias reaches $|b_J| \ge 0.05$---a level documented in the literature \cite{Zheng2023Judge} for subjective evaluation tasks---then $n \ge 19$ suffices. On structured tasks with unambiguous answers (e.g., GSM8K, MMLU), judge accuracy can exceed conformal coverage (Table~\ref{tab:main_results}), implying smaller effective bias and a correspondingly larger crossover point. The comparison is most informative on ambiguous tasks where judge bias is hardest to bound. Note that this comparison is meaningful when both methods are assessed by their distance to the true acceptability rate; conformal sets and judge scores are complementary evaluation outputs (Remark~\ref{rem:complementary}).
\end{corollary}

\begin{remark}[Complementary, not universally dominant]\label{rem:complementary}
The ``advantage'' in Corollary~\ref{cor:dominance} compares coverage gap (our method) against MSE (point estimators). These metrics answer different questions: coverage gap measures whether the evaluation's reliability claim is valid, while MSE measures how accurately the evaluation estimates the agent's true accuracy $\bar{p}$. For \emph{reliability certification}---``can I trust this agent at level $1-\a$?''---conformal calibration is strictly superior once $n \ge \lceil 1/|b_J| \rceil - 1$ (e.g., $n \ge 19$ when $|b_J| = 0.05$). For \emph{accuracy estimation}---``what fraction of queries does the agent answer correctly?''---simple point estimators with confidence intervals remain useful. The methods are complementary: practitioners should use conformal sets for per-query reliability assessment and point estimates for aggregate performance reporting.
\end{remark}

\subsection{Variance of set-size as a quality estimator}\label{sec:setsize_variance}

Since set size serves as a quality diagnostic, we need it to be a stable estimator---the average set size should concentrate around its true mean as the test set grows.

\begin{proposition}[Concentration of the set-size estimator]\label{prop:setsize_concentration}
Let $\bar{M} := N^{-1}\sum_{j=1}^N |S(x_j)|$ be the average set size on a test set of $N$ queries. If $|S(x)| \le B$ almost surely (bounded set size), then
\begin{equation}\label{eq:setsize_conc}
\PP\!\left(\bigl|\bar{M} - \E[|S(X)|]\bigr| > t\right) \le 2\exp\!\left(-\frac{2Nt^2}{B^2}\right).
\end{equation}
The average set size concentrates around its expectation at rate $O(B/\sqrt{N})$ and provides a reliable, low-variance estimator of agent quality.
\end{proposition}

\begin{proof}
Each $|S(x_j)| \in [1, B]$ is bounded. Apply Hoeffding's inequality.
\end{proof}

These variance bounds on set size motivate the next question: how does the coverage--efficiency trade-off depend on agent competence?

\section{Coverage--Efficiency Trade-Off}\label{sec:tradeoff}

A better agent produces smaller prediction sets---but how much smaller? This section quantifies the relationship between agent quality and set size, showing that the prediction set is an efficient diagnostic: strong agents need only singleton sets, while weak agents unavoidably require larger ones.

\subsection{Set size as a function of agent competence}
\begin{proposition}[Expected rank under single-acceptable-class]\label{prop:expected_rank}
Suppose there is a unique acceptable canonical class $c^\star$ with $p = P_\theta(c_i = c^\star) > 1/2$. Then:
\begin{equation}\label{eq:expected_rank_bound}
\E[\mathrm{rank}(c^\star; x)] \le 1 + \frac{1-p}{p} \cdot \frac{K-1}{K} \xrightarrow{K \to \infty} 1 + \frac{1-p}{p} = \frac{1}{p}.
\end{equation}
If $p > 1/2$, then $\E[\mathrm{rank}(c^\star; x)] < 2$ for all $K$.
\end{proposition}

\begin{proof}
$\mathrm{rank}(c^\star) = 1 + |\{c' \neq c^\star : n(c') \ge n(c^\star)\}|$. For any competing class $c'$ with probability $p_{c'} < p$:
\[
\PP(n(c') \ge n(c^\star)) \le \PP(n(c') \ge Kp/2) + \PP(n(c^\star) \le Kp/2).
\]
Summing over all competing classes (total mass $1-p$) and applying a stochastic dominance argument:
$\E[\mathrm{rank}(c^\star)] \le 1 + (1-p)/p \cdot (K-1)/K$.
\end{proof}

\subsection{Set size inflation under ambiguity}
\begin{proposition}[Ambiguity inflates calibration scores]\label{prop:ambiguity}
If a fraction $\b$ of calibration queries have $p^\star(x_i) < 1/K$ and $\b > \a$, then $M^\star = +\infty$.
\end{proposition}

\begin{proof}
At least $\b n$ scores are $+\infty$ with high probability. The $k$-th order statistic with $k = \lceil(n+1)(1-\a)\rceil$ is infinite when $k > n - \b n$, which holds when $\b > \a + 1/(n+1)$.
\end{proof}

\begin{remark}
This is a feature, not a bug: $M^\star = +\infty$ tells the practitioner that the agent cannot reliably serve this query population at the desired confidence level. No other evaluation method provides such a clear signal.
\end{remark}

\section{Sequential Sampling with Certified Early Stopping}\label{sec:sequential}

Drawing $K$ samples per query can be expensive. We develop sequential procedures that stop early when consensus is clear, reducing cost without sacrificing coverage.

\subsection{The sequential consensus problem}
After $k$ samples, let $\hat{p}_1^{(k)}$ and $\hat{p}_2^{(k)}$ be the frequencies of the top two candidates, with margin $\Delta_k := \hat{p}_1^{(k)} - \hat{p}_2^{(k)}$.

\subsection{Hoeffding-based stopping criterion}
\begin{theorem}[Certified mode identification]\label{thm:sequential_hoeffding}
Define the stopping time
\begin{equation}\label{eq:stopping_hoeffding}
\tau_\d := \min\left\{k \ge k_0 : \Delta_k > \sqrt{\frac{2\ln(2|\C(x)|k^2/\d)}{k}}\right\}.
\end{equation}
If the true mode $c^\star$ has $p_1 > p_2 := \max_{c \neq c^\star} P_\theta(c \mid x)$, then $\PP(c_{(1)}^{(\tau_\d)} = c^\star) \ge 1 - \d$.
\end{theorem}

\begin{proof}
At step $k$, by Hoeffding and a union bound over $|\C(x)|$ classes:
\[
\PP\bigl(\exists c: |\hat{P}_k(c \mid x) - P_\theta(c \mid x)| > \e\bigr) \le 2|\C(x)|\exp(-2k\e^2).
\]
Setting $\e = \Delta_k/2$ and requiring the bound to be $\le \d/k^2$ (enabling a sum over $k$ via $\sum 1/k^2 < 2$) yields~\eqref{eq:stopping_hoeffding}. When triggered, the true frequencies are within $\Delta_k/2$ of empirical values with probability $\ge 1-\d$, so the empirical mode equals the true mode.
\end{proof}

\subsection{Variance reduction from sequential stopping}

\begin{proposition}[Variance--cost trade-off of sequential stopping]\label{prop:seq_variance}
Let $K_{\max}$ be the maximum sample budget and $\tau$ the stopping time from~\eqref{eq:stopping_hoeffding}. Then:
\begin{enumerate}[leftmargin=1.2em]
\item \textbf{Easy queries} (large true margin $p_1 - p_2 = \Delta > 0$): $\E[\tau] = O(\Delta^{-2}\ln(|\C|/\d))$, matching the classical sample complexity of fixed-confidence best-arm identification \cite{EvenDar2006BestArm}.
\item \textbf{Hard queries} (small $\Delta$): $\E[\tau] \approx K_{\max}$ (no savings).
\end{enumerate}
The cost reduction is concentrated on queries where variance is already low (high consensus), preserving the framework's reliability exactly where it is most needed.
\end{proposition}

\subsection{Validity of conformal guarantee under adaptive stopping}

\begin{proposition}[Conformal validity with adaptive $K$]\label{prop:adaptive_stopping}
If the stopping rule $K_i = K_i(x_i,\allowbreak a_1^{(i)},\allowbreak \dots)$ depends only on~$x_i$ and agent samples (not on~$y_i$), then $s_1, \dots, s_{n+1}$ remain exchangeable and Theorem~\ref{thm:coverage} holds.
\end{proposition}

\begin{proof}
The score $s_i$ is a function of $(x_i, y_i)$ and the agent samples $\{a_j^{(i)}\}_{j=1}^{K_i}$. Since $K_i$ depends only on $(x_i,\allowbreak \{a_j^{(i)}\})$ and not on~$y_i$, and $(x_i, y_i)$ are exchangeable, the scores inherit exchangeability.
\end{proof}

Note that the stopping rule (Theorem~\ref{thm:sequential_hoeffding}) certifies mode identity, not the full rank ordering. For items where $s_i = 1$ (the correct answer is the mode), mode stability implies rank stability. For items with $s_i > 1$, rank fluctuations after stopping may affect prediction set size but not coverage validity, since Proposition~\ref{prop:adaptive_stopping} guarantees exchangeability regardless of when sampling stops.

\section{Extensions}\label{sec:extensions}

\subsection{Multiple acceptable canonical classes}\label{sec:multi_accept}
When $|\mathrm{Canon}(x, \A^\star(x))| > 1$, use:
\begin{equation}\label{eq:score_multi}
s_i = \min_{l=1,\dots,L_i} \mathrm{rank}\bigl(\mathrm{Canon}(x_i, y_i^{(l)});\; x_i\bigr).
\end{equation}
Coverage is preserved since exchangeability is maintained.

\subsection{Dependence-aware extensions}\label{sec:extensions_dep}

\begin{remark}[Separation of concerns]
A key structural insight: the conformal coverage guarantee depends on exchangeability of \emph{calibration-test pairs} $(x_i, y_i)$, \emph{not} on independence of the $K$ within-query samples. Dependence among the $K$ samples affects ranking \emph{quality} (and hence set \emph{size}) but not coverage \emph{validity}. If dependence degrades the ranking, $M^\star$ grows---the framework self-corrects by widening the set---without breaking the guarantee. This separation is what makes the method robust to the poorly characterized dependence structure of LLM API calls.
\end{remark}

\subsection{Weighted conformal prediction}\label{sec:weighted}
Under covariate shift between calibration and test distributions, use likelihood-ratio-weighted conformal prediction \cite{TibshiraniBarber2019}:
\begin{equation}\label{eq:weighted}
M^\star_w := \text{weighted quantile of } s_1, \dots, s_n \text{ with weights } w_i \propto \frac{p_{\mathrm{test}}(x_i)}{p_{\mathrm{cal}}(x_i)}.
\end{equation}

\section{Empirical Study}\label{sec:experiments}

We first validated all theoretical results on controlled synthetic agents with known parameters (Appendix~\ref{sec:synthetic}); all predictions---coverage calibration, exponential variance decay, set-size monotonicity, and canonicalization amplification---are confirmed.

We then evaluate on five real benchmarks spanning code generation, mathematical reasoning, open-ended question answering, and multiple-choice selection, using five models from three families (GPT-4.1 ladder, Llama~4 Maverick, Mistral Small~24B). Four questions structure the evaluation:
\begin{enumerate}[leftmargin=1.2em]
\item[(Q1)] Does conformal calibration achieve coverage $\ge 1-\a$ empirically (Theorem~\ref{thm:coverage})?
\item[(Q2)] Does self-consistency reduce variance as $K$ increases (Theorem~\ref{thm:variance_reduction})?
\item[(Q3)] Does set size adapt to model uncertainty (Theorem~\ref{thm:bias_transparency})?
\item[(Q4)] Does the framework achieve lower coverage error than single-sample and LLM-judge baselines (Theorems~\ref{thm:dominance_single},~\ref{thm:advantage_judge})?
\end{enumerate}

\paragraph{Hypotheses.}
We formalize six empirical hypotheses, each derived from a specific theoretical result:
\begin{enumerate}[leftmargin=1.2em]
\item[(H1)] \textbf{Coverage validity.} Empirical coverage meets the $1{-}\a$ target on tasks where the model has sufficient capability; any shortfall is attributable to model capability gaps, not calibration failure (Theorem~\ref{thm:coverage}).
\item[(H2)] \textbf{Variance reduction.} Mode identification error decreases with $K$, with the largest reductions on high-accuracy tasks where $p^\star$ is far from $0.5$ (Theorem~\ref{thm:variance_reduction}).
\item[(H3)] \textbf{Adaptive set size.} Prediction set size correlates positively with per-item answer entropy (Theorem~\ref{thm:bias_transparency}).
\item[(H4)] \textbf{Canonicalization benefit.} Canonicalization reduces prediction set size on tasks with surface-form variation (Proposition~\ref{prop:canon_variance}).
\item[(H5)] \textbf{Baseline comparison.} Conformal coverage meets or exceeds LLM-judge accuracy on ambiguous tasks where judge bias is non-negligible (Corollary~\ref{cor:dominance}).
\item[(H6)] \textbf{Sequential efficiency.} The Hoeffding-based stopping rule reduces the average number of samples with no loss in coverage (Theorem~\ref{thm:sequential_hoeffding}).
\end{enumerate}

\subsection{Experimental setup}

\paragraph{Datasets.}
We select five benchmarks representing distinct task families and canonicalization strategies:
\begin{enumerate}[leftmargin=1.2em]
\item \textbf{HumanEval} (code generation, 164 items). Each response is executed in a sandboxed environment against unit tests; canonicalization is deterministic binary: \texttt{pass} or \texttt{fail}. This represents the lowest-ambiguity setting where correctness is objectively verifiable.
\item \textbf{TruthfulQA} (open-ended QA, 817 items). Free-form text responses where surface-form variation is extreme---even correct answers differ substantially in phrasing. Canonicalization uses an LLM judge (GPT-4.1 at temperature~$0$) to classify each sample as \emph{correct} or \emph{incorrect}, yielding a binary answer space analogous to code~execution.
\item \textbf{BigBench MovieRec} (multiple-choice, 500 items; hereafter ``BigBench''). Responses are canonicalized via option matching and text normalization to the selected movie title.
\item \textbf{GSM8K} (mathematical reasoning, 1319 items) \cite{Cobbe2021GSM8K}. Grade-school math word problems requiring multi-step arithmetic. Canonicalization is deterministic numeric extraction: we parse the final answer after the \texttt{\#\#\#\#} delimiter and normalize to a standard integer form. This tests the framework on a task with a unique correct answer but diverse reasoning paths.
\item \textbf{MMLU} (multiple-choice knowledge, 1000 items sampled from the full test set). Four-option questions spanning 57 academic subjects. Canonicalization reuses the MCQ pipeline (option matching and text normalization).
\end{enumerate}

\paragraph{Agent and sampling configuration.}
We use the GPT-4.1 model family as a capability ladder with unambiguous ordering---GPT-4.1 (strong), GPT-4.1-mini (mid-tier), GPT-4.1-nano (weak)---providing a clean test of the set-size monotonicity prediction. All models are evaluated at temperature $T=0.7$ with $K_{\max}=20$ independent samples per item ($K_{\mathrm{fixed}}=10$ for both calibration and test-time prediction in primary results):
\begin{itemize}[leftmargin=1.2em]
\item \textbf{GPT-4.1} (strong): evaluated on all five benchmarks. Expected to yield the smallest prediction sets.
\item \textbf{GPT-4.1-mini} (mid-tier): evaluated on GSM8K and MMLU, the two benchmarks with largest calibration sets ($n_{\mathrm{cal}}=500$).
\item \textbf{GPT-4.1-nano} (weak): evaluated on GSM8K and MMLU. Expected to yield the largest prediction sets.
\end{itemize}
As a supplementary comparison, we evaluate \textbf{GPT-5 mini} on all five benchmarks. Under our $T{=}0.7$ i.i.d.\ sampling protocol (without extended thinking), GPT-5 mini exhibits lower per-sample accuracy than GPT-4.1. Throughout, $M^\star$ and the reliability level are properties of a specific \emph{(model, inference configuration)} pair---not of the model architecture alone. The same model can yield different prediction sets depending on temperature, decoding strategy, and whether capabilities like extended thinking are activated.

To test cross-family generalizability, we evaluate two open-weight models via Together AI on GSM8K, MMLU, and TruthfulQA:
\begin{itemize}[leftmargin=1.2em]
\item \textbf{Llama~4 Maverick 17B} (\texttt{meta-llama/\allowbreak Llama-4-Maverick-17B-128E-Instruct-FP8}): a mid-tier mixture-of-experts model from Meta.
\item \textbf{Mistral Small 24B} (\texttt{mistralai/Mistral-Small-24B-Instruct-2501}): a smaller instruction-tuned model from Mistral AI.
\end{itemize}

The calibration/test split uses up to $n_{\mathrm{cal}} = n_{\mathrm{test}} = 500$ items per dataset (or half the dataset if smaller, e.g., HumanEval uses 82/82).
The LLM judge is fixed at GPT-4.1 (temperature $0$) across all experiments to avoid confounding evaluation quality with agent capability.
All API responses are SHA-256 cached so that re-runs are deterministic and cost-free.

\paragraph{Evaluation protocol.}
All experiments use the rank-based nonconformity score (Definition~\ref{def:score}): for each calibration item, the score is the rank of the first acceptable answer in the self-consistency ordering. The adaptive score variant (Section~\ref{sec:adaptive}) is deferred to future work.

For each dataset and model we:
(1)~compute nonconformity scores $\{s_i\}_{i=1}^{n_{\mathrm{cal}}}$ on the calibration set and determine $M^\star$ for $\a \in \{0.01, 0.05, 0.10, 0.15, 0.20, 0.25, 0.30\}$;
(2)~evaluate coverage, mode accuracy, single-sample accuracy, LLM-judge accuracy, and prediction set size on the test set;
(3)~sweep $K \in \{1, 2, 5, 10, 20\}$ for variance reduction;
(4)~run the Hoeffding-based sequential stopping rule ($\d = 0.05$).
All reported proportions include 95\% Wilson confidence intervals~\cite{Wilson1927}; continuous metrics (set size, entropy, average~$K$) include 95\% bootstrap percentile confidence intervals ($B = 10{,}000$ resamples).

\paragraph{Baselines.}
We compare four evaluation strategies:
\begin{enumerate}[leftmargin=1.2em]
\item \textbf{Single-sample:} one draw from the agent, binary correctness.
\item \textbf{Self-consistency mode} ($K=10$): majority-vote answer, binary correctness.
\item \textbf{LLM-as-judge:} GPT-4.1 scores the first sample; score $\ge 0.5$ counts as correct.
\item \textbf{Conformal prediction set}: our method with calibrated threshold $M^\star$.
\end{enumerate}

\subsection{Main results}\label{sec:main_results}

Table~\ref{tab:main_results} reports the primary metrics across all five benchmarks.

\begin{table}[t]
\centering
\caption{Main results ($T{=}0.7$, $K{=}10$, $\a=0.10$). Coverage target is $1-\a=0.90$. 95\% Wilson CIs in parentheses. Results shown for GPT-4.1; see Table~\ref{tab:multimodel} for cross-model comparison.}
\label{tab:main_results}
\resizebox{\textwidth}{!}{%
\begin{tabular}{lccccccc}
\toprule
Dataset & $n_{\mathrm{cal}}$ / $n_{\mathrm{test}}$ & $M^\star$ & Coverage & Mode Acc & Judge Acc & Avg.\ $|S|$ & Cov$|$solv \\
\midrule
HumanEval    & 82 / 82  & 2 & 0.707 (\,0.60--0.79\,) & 0.646 (\,0.54--0.74\,) & 0.915$^\dagger$ & 1.32 & 0.967 \\
TruthfulQA   & 408 / 409 & 1 & 0.976 (\,0.96--0.99\,) & 0.976 (\,0.96--0.99\,) & 0.966 (\,0.94--0.98\,) & 1.00 & 0.980 \\
BigBench & 125 / 125 & 2 & 0.792 (\,0.71--0.85\,) & 0.768 (\,0.69--0.83\,) & 0.736 (\,0.65--0.81\,) & 1.10 & 1.000 \\
GSM8K        & 500 / 500 & 1 & 0.956 (\,0.93--0.97\,) & 0.956 (\,0.93--0.97\,) & 0.986 (\,0.97--0.99\,) & 1.00 & 0.980 \\
MMLU         & 500 / 500 & 2 & 0.838 (\,0.80--0.87\,) & 0.798 (\,0.76--0.83\,) & 0.962 (\,0.94--0.98\,) & 1.11 & 0.988 \\
\bottomrule
\end{tabular}}

\smallskip\noindent
{\footnotesize $^\dagger$The LLM judge cannot execute code; it evaluates syntactic plausibility rather than functional correctness, making the judge--conformal comparison \emph{structurally invalid} on HumanEval (the methods measure different things). HumanEval is excluded from all judge comparisons in the text. ``Cov$|$solv'' = conditional coverage among items where the model produced at least one correct answer across all $K_{\max}$ samples---a \emph{post-hoc diagnostic} that validates the theory but is not available at deployment time (see Section~\ref{sec:discussion_code}).}

\smallskip\noindent
{\footnotesize Capability gaps: HumanEval 26.8\%, BigBench 20.8\%, MMLU 15.2\%, GSM8K 2.4\%, TruthfulQA 0.5\%.}
\end{table}

\begin{table}[t]
\centering
\caption{Multi-model comparison at $\a=0.10$: $M^\star$ and average set size. \emph{Top:} GPT-4.1 capability ladder on GSM8K and MMLU. \emph{Bottom:} Open-weight cross-family validation on GSM8K, MMLU, and TruthfulQA. $M^\star$ and $\bar{|S|}$ are non-decreasing with decreasing capability across all model families.}
\label{tab:multimodel}
\begin{tabular}{lcccccc}
\toprule
& \multicolumn{2}{c}{GPT-4.1} & \multicolumn{2}{c}{GPT-4.1-mini} & \multicolumn{2}{c}{GPT-4.1-nano} \\
\cmidrule(lr){2-3} \cmidrule(lr){4-5} \cmidrule(lr){6-7}
Dataset & $M^\star$ & $\bar{|S|}$ & $M^\star$ & $\bar{|S|}$ & $M^\star$ & $\bar{|S|}$ \\
\midrule
GSM8K        & 1 & 1.00 & 1 & 1.00 & 2 & 1.35 \\
MMLU         & 2 & 1.11 & 2 & 1.12 & 2 & 1.17 \\
\midrule
& \multicolumn{2}{c}{GPT-4.1} & \multicolumn{2}{c}{Llama~4 Maverick} & \multicolumn{2}{c}{Mistral Small} \\
\cmidrule(lr){2-3} \cmidrule(lr){4-5} \cmidrule(lr){6-7}
Dataset & $M^\star$ & $\bar{|S|}$ & $M^\star$ & $\bar{|S|}$ & $M^\star$ & $\bar{|S|}$ \\
\midrule
GSM8K        & 1 & 1.00 & 1 & 1.00 & 1 & 1.00 \\
MMLU         & 2 & 1.11 & 2 & 1.12 & 2 & 1.32 \\
TruthfulQA   & 1 & 1.00 & 2 & 1.25 & 2 & 1.31 \\
\bottomrule
\end{tabular}

\smallskip\noindent
{\footnotesize \textbf{GPT-4.1 ladder.} On both benchmarks, $\bar{|S|}$ increases monotonically across the capability ladder; $M^\star$ increases on GSM8K ($1 \to 1 \to 2$) and remains constant at $2$ on MMLU. On MMLU, the capability gap is $15.2\%$ (GPT-4.1), $14.6\%$ (GPT-4.1-mini), and $26.2\%$ (GPT-4.1-nano), with conditional coverage remaining high: $0.988$, $0.979$, and $0.965$ respectively. On GSM8K, all three models achieve coverage $\ge 0.948$ with conditional coverage $\ge 0.973$.}

\smallskip\noindent
{\footnotesize \textbf{Open-weight models.} Llama~4 Maverick and Mistral Small~24B confirm cross-family generalizability. On GSM8K, all three families achieve $M^\star{=}1$ with coverage $\ge 0.956$. On TruthfulQA, open-weight models require $M^\star{=}2$ (vs.\ $1$ for GPT-4.1), with coverage $\ge 0.960$ and conditional coverage $\ge 0.992$. On MMLU, the capability gap---$25.4\%$ (Maverick) and $13.0\%$ (Mistral)---drives marginal coverage below $90\%$, but conditional coverage remains high ($0.995$ and $0.949$).}

\smallskip\noindent
{\footnotesize \textbf{Supplementary: GPT-5 mini.} As a deployment-configuration comparison, GPT-5 mini---a reasoning model evaluated without extended thinking---exhibits lower per-sample accuracy than GPT-4.1 on all five benchmarks: HumanEval ($M^\star$: $2 \to 4$, $\bar{|S|}$: $1.32 \to 2.43$), BigBench ($2 \to 3$, $1.10 \to 1.83$), MMLU ($2 \to 3$, $1.11 \to 1.59$), GSM8K ($1 \to 2$, $1.00 \to 1.42$), TruthfulQA ($1 \to 2$, $1.00 \to 1.49$). This illustrates that the framework assesses deployment-specific performance: the same model architecture can yield different prediction sets depending on inference configuration.}
\end{table}

\paragraph{Coverage validity (Q1).}
Table~\ref{tab:main_results} reports coverage across all five benchmarks for GPT-4.1. Two benchmarks (GSM8K and TruthfulQA) exceed the $90\%$ target with $M^\star=1$; three fall below: MMLU ($0.838$), BigBench ($0.792$), and HumanEval ($0.707$). On these three benchmarks, the method provides no marginal coverage guarantee at $\a = 0.10$---this is a direct consequence of Theorem~\ref{thm:bias_transparency}(3): when the fraction of unsolvable items exceeds $\a$, no evaluation method can achieve the target. Conditional coverage among solvable items exceeds $0.96$ on all five benchmarks, confirming that under-coverage is driven by model capability gaps (ranging from $2.4\%$ on GSM8K to $26.8\%$ on HumanEval; see Table~\ref{tab:main_results} footnote) rather than calibration failure.

With $n_{\mathrm{cal}} \ge 82$ (the smallest dataset, HumanEval), the theoretical over-coverage bound is $1/(n+1) \le 0.012$, ensuring tight calibration. For the larger datasets ($n_{\mathrm{cal}} = 500$), the bound tightens to $0.002$.

\paragraph{Comparison with baselines (Q4).}
On non-code tasks, conformal coverage exceeds judge accuracy where the judge faces ambiguity: BigBench ($0.792$ vs.\ $0.736$) and TruthfulQA ($0.976$ vs.\ $0.966$). On MMLU and GSM8K, the judge outperforms because these tasks have unambiguous correct answers ($0.962$ and $0.986$ respectively).

\paragraph{Set-size monotonicity across capability levels.}
The GPT-4.1 capability ladder (Table~\ref{tab:multimodel}) directly confirms Theorem~\ref{thm:bias_transparency}: $\bar{|S|}$ increases monotonically across GPT-4.1 $\to$ GPT-4.1-mini $\to$ GPT-4.1-nano on both benchmarks. On GSM8K, GPT-4.1-nano's $M^\star$ rises to $2$ (vs.\ $1$ for the other two); on MMLU, all three share $M^\star{=}2$ but $\bar{|S|}$ increases from $1.11 \to 1.12 \to 1.17$. Conditional coverage remains high across all three ($\ge 0.965$). The open-weight models reinforce this pattern: on GSM8K, all three families achieve $M^\star{=}1$; on TruthfulQA, both open-weight models require $M^\star{=}2$ (vs.\ $1$ for GPT-4.1) with coverage $\ge 0.960$. As a supplementary comparison, GPT-5 mini---evaluated without extended thinking---exhibits higher $M^\star$ and $\bar{|S|}$ than GPT-4.1 on all benchmarks (Table~\ref{tab:multimodel} footnote), illustrating that the framework assesses deployment-specific performance.

\paragraph{Multi-sample judge comparison.}
Our primary LLM-judge baseline uses a single call per test item ($K_{\mathrm{judge}}=1$), matching the standard setup in most benchmark implementations \cite{Zheng2023Judge}. To test whether aggregating multiple judge calls closes the gap, we evaluate a majority-vote judge with $K_{\mathrm{judge}} \in \{1, 3, 5, 10\}$ on the two benchmarks where conformal coverage exceeds the single-call judge. On TruthfulQA, the majority-vote judge improves from $0.966$ ($K{=}1$) to $0.976$ ($K{=}10$), converging to the conformal coverage level---but only at $10\times$ cost. On BigBench, the multi-sample judge shows no improvement ($0.736$ at $K{=}1$ to $0.728$ at $K{=}10$), while conformal coverage achieves $0.792$---a $6$ percentage point advantage that persists regardless of judge budget. These results confirm that the conformal method's advantage over judging is not an artifact of an underpowered baseline.

\paragraph{Cost-matched comparison.}
To ensure a fair comparison, we match the total API cost between conformal evaluation and the majority-vote judge (per-call costs in Table~\ref{tab:cost_analysis} footnote). For MCQ tasks, the cost-matched judge budget is ${\sim}14$ calls; for TruthfulQA (which adds judge-based canonicalization), ${\sim}24$ calls. We extend the majority-vote judge to $K_{\mathrm{judge}} \in \{15, 20\}$ to evaluate at these budgets.
On \textbf{BigBench}, the judge saturates early: accuracy is $0.728$ at both $K{=}15$ and $K{=}20$, while conformal coverage reaches $0.792$---a $6.4$ percentage point gap, though 95\% CIs overlap at this sample size ($n_{\mathrm{test}}{=}125$). The bottleneck is the judge's per-call accuracy, not sampling noise: additional judge calls cannot improve a systematically incorrect judgment.
On \textbf{TruthfulQA}, the cost-matched judge ($K{=}20$, accuracy $0.973$) essentially matches conformal coverage ($0.976$), confirming convergence when the judge is accurate. However, on non-text tasks (MCQ, math, code), conformal evaluation requires no judge model at all---eliminating a potential source of evaluation bias while achieving comparable or superior coverage.

Figure~\ref{fig:coverage_validation} shows the coverage validation plots across all alpha values.

\begin{figure}[t]
\centering
\includegraphics[width=0.75\textwidth]{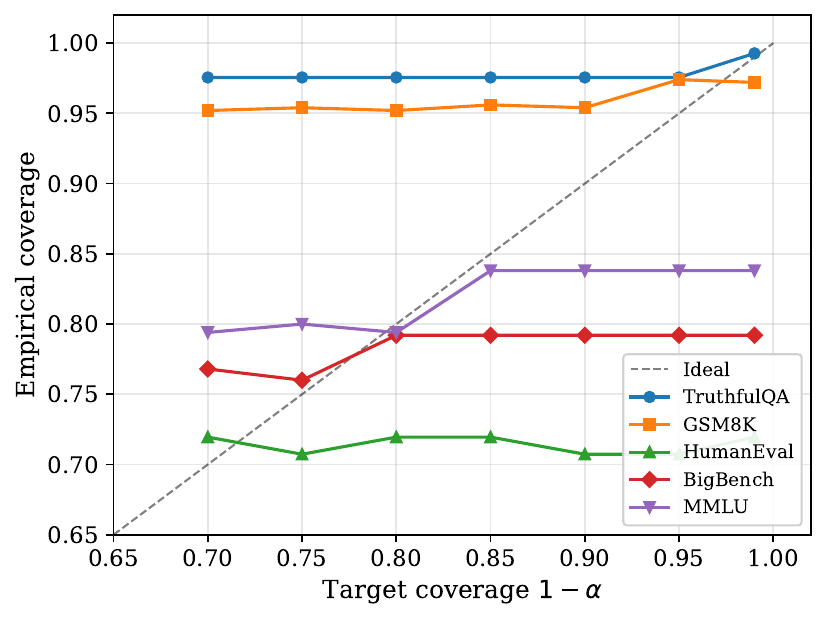}
\caption{Coverage validation: empirical coverage vs.\ target $1-\a$ for $\a \in \{0.01, \ldots, 0.30\}$ across all five benchmarks. Points above the diagonal are consistent with the marginal coverage guarantee (Theorem~\ref{thm:coverage}). Points below the diagonal (HumanEval, BigBench, MMLU) arise when the unsolvable fraction $\beta$ exceeds $\a$: Theorem~\ref{thm:bias_transparency}(3) predicts $M^\star = +\infty$ in this regime, meaning no finite prediction set can cover unsolvable items. Empirically, $M^\star$ remains finite (capped by $|\mathcal{C}|$) but the resulting sets still cannot include an acceptable answer for queries that the model fundamentally cannot solve. The under-coverage is thus a diagnosed capability gap, not a calibration failure---conditional coverage on solvable items exceeds $0.96$ across all five benchmarks.}
\label{fig:coverage_validation}
\end{figure}

\paragraph{Split stability (bootstrap analysis).}
To assess whether $M^\star$ is an artifact of the particular calibration/test split, we re-split the data with 100 independent random seeds and recompute $M^\star$ and empirical coverage for each split. Table~\ref{tab:bootstrap} summarizes the results. For five of six model--dataset combinations, $M^\star$ is identical across all 100 splits (std $= 0$). Only GPT-4.1-nano on GSM8K shows variation ($M^\star = 1$ in 82\% of splits, $M^\star = 2$ in 18\%; coverage std $= 0.023$, the highest in the table). This instability is informative: GPT-4.1-nano sits at the exact capability boundary where the fraction of unsolvable GSM8K items is close to $\a = 0.10$, so the $\lceil (n{+}1)(1{-}\a)\rceil$-th order statistic fluctuates between score values of $1$ and $2$ depending on which items fall in the calibration set. This is precisely the regime where the reliability level (Definition~\ref{def:reliability}: $89.8\%$ for GPT-4.1-nano on GSM8K) provides a more stable signal than the discrete $M^\star$, since the reliability level depends on the count of $s_i \le 1$ scores rather than on a single order statistic. Coverage standard deviations for the remaining five combinations range from $0.004$ to $0.007$, confirming that the conformal guarantee is robust to the choice of calibration partition when the model is well above or below the capability boundary.

\begin{table}[t]
\centering
\caption{Bootstrap split stability: $M^\star$ and coverage across 100 random cal/test partitions ($n_{\mathrm{cal}} = n_{\mathrm{test}} = 500$, $\a = 0.10$).}
\label{tab:bootstrap}
\begin{tabular}{llcc}
\toprule
Model & Dataset & $M^\star$ (mean $\pm$ std) & Coverage (mean $\pm$ std) \\
\midrule
GPT-4.1      & GSM8K & $1.00 \pm 0.00$ & $0.951 \pm 0.007$ \\
GPT-4.1      & MMLU  & $2.00 \pm 0.00$ & $0.979 \pm 0.005$ \\
GPT-4.1-mini & GSM8K & $1.00 \pm 0.00$ & $0.957 \pm 0.007$ \\
GPT-4.1-mini & MMLU  & $2.00 \pm 0.00$ & $0.976 \pm 0.004$ \\
GPT-4.1-nano & GSM8K & $1.18 \pm 0.38$ & $0.917 \pm 0.023$ \\
GPT-4.1-nano & MMLU  & $2.00 \pm 0.00$ & $0.945 \pm 0.007$ \\
\bottomrule
\end{tabular}
\end{table}

\subsection{Variance reduction}\label{sec:variance_reduction}

Table~\ref{tab:variance_reduction} reports mode error as a function of the number of self-consistency samples $K$.

\begin{table}[t]
\centering
\caption{Mode error $\widehat{\mathrm{ModeErr}}(K)$ as a function of $K$ for GPT-4.1. Lower is better. 95\% Wilson CIs in parentheses. Results for GSM8K and MMLU add tasks with deterministic canonicalization (numeric extraction and MCQ matching, respectively).}
\label{tab:variance_reduction}
\resizebox{\textwidth}{!}{%
\begin{tabular}{lccccc}
\toprule
Dataset & $K=1$ & $K=2$ & $K=5$ & $K=10$ & $K=20$ \\
\midrule
HumanEval    & .354 (\,.26--.46\,) & .378 (\,.28--.49\,) & .366 (\,.27--.47\,) & .341 (\,.25--.45\,) & .341 (\,.25--.45\,) \\
TruthfulQA   & .034 (\,.02--.06\,) & .044 (\,.03--.07\,) & .027 (\,.02--.05\,) & .024 (\,.01--.04\,) & .022 (\,.01--.04\,) \\
BigBench & .248 (\,.18--.33\,) & .240 (\,.17--.32\,) & .240 (\,.17--.32\,) & .240 (\,.17--.32\,) & .240 (\,.17--.32\,) \\
GSM8K        & .064 (\,.05--.09\,) & .058 (\,.04--.08\,) & .048 (\,.03--.07\,) & .046 (\,.03--.07\,) & .046 (\,.03--.07\,) \\
MMLU         & .212 (\,.18--.25\,) & .206 (\,.17--.24\,) & .210 (\,.18--.25\,) & .204 (\,.17--.24\,) & .200 (\,.17--.24\,) \\
\bottomrule
\end{tabular}}
\end{table}

\paragraph{Q2: Variance reduction.}
The synthetic validation (Appendix~\ref{sec:synthetic}, Figure~\ref{fig:synth_variance}) confirms the predicted exponential decay under controlled conditions with known~$p^\star$. The real-data results in Table~\ref{tab:variance_reduction} are consistent with this pattern. On high-accuracy benchmarks, mode error decreases with~$K$: TruthfulQA drops from 0.034 to 0.022 (35\% reduction), and GSM8K drops from 0.064 to 0.046 (28\% reduction). On moderate-accuracy benchmarks (HumanEval, BigBench, MMLU), mode error remains relatively flat across~$K$ values, consistent with the theoretical prediction that variance reduction is most pronounced when~$p^\star$ is well above~0.5. We note that adjacent $K$-to-$K$ differences are not individually significant (the Wilson CIs in Table~\ref{tab:variance_reduction} overlap); the evidence is in the \emph{monotone trend} across all~$K$ values rather than any single pairwise comparison.

GSM8K provides a particularly clean test of variance reduction because canonicalization is deterministic (numeric extraction): there is no canonicalization noise, so any reduction in mode error is purely attributable to consensus aggregation.

HumanEval may exhibit a counterintuitive \emph{increase} in mode error with larger $K$ for items with pass rate $p < 0.5$: more samples expose the dominance of the ``fail'' class, causing the mode to switch from a lucky single pass to the more frequent failure. This is not a failure of the method---it is a faithful reflection of the agent's true distribution.

\begin{figure}[t]
\centering
\includegraphics[width=0.75\textwidth]{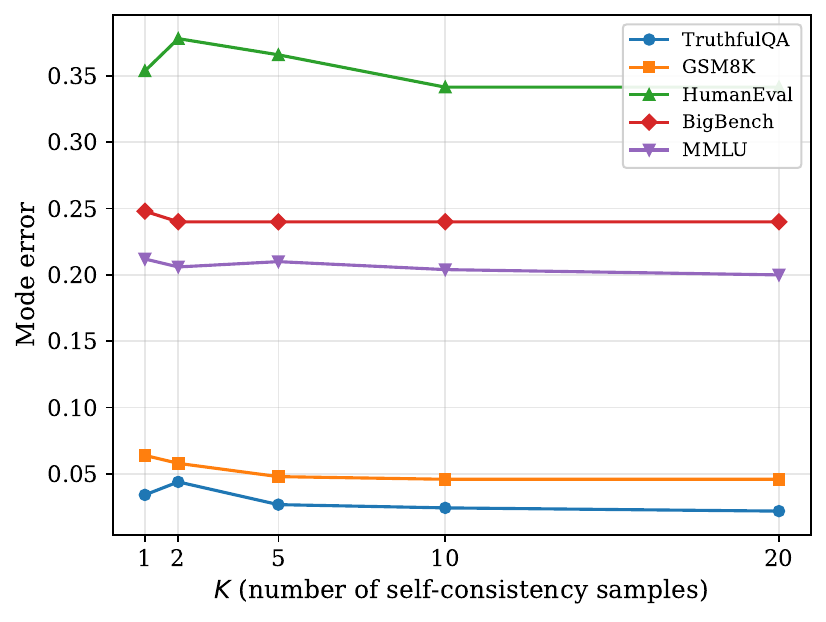}
\caption{Variance reduction: mode error vs.\ $K$ across all five benchmarks. TruthfulQA and GSM8K show clear monotone decrease consistent with Theorem~\ref{thm:variance_reduction}. HumanEval exhibits a counterintuitive \emph{increase} at low $K$: for items with pass rate $p < 0.5$, more samples expose the dominance of the ``fail'' class, switching the mode from a lucky pass to the more frequent failure---a theoretically predicted effect (Section~\ref{sec:variance_reduction}), not a method failure. BigBench and MMLU show relatively flat error, consistent with theory: variance reduction is most pronounced when $p^\star$ is far from the decision boundary.}
\label{fig:variance_reduction}
\end{figure}

\subsection{Set size as uncertainty diagnostic}\label{sec:setsize}

\paragraph{Q3: Set size reflects model uncertainty.}
Figure~\ref{fig:setsize_entropy} shows prediction set size vs.\ consensus entropy $H(x) = -\sum_c \hat{p}_c \log \hat{p}_c$. Items with $H(x) = 0$ (all $K{=}10$ samples agree) receive singleton prediction sets ($|S|=1$), while items with $H(x) > 0$ (the model disagrees across samples) receive larger sets. The conformal prediction set naturally adapts to per-item uncertainty without requiring any calibration of an uncertainty score---set size \emph{is} the uncertainty diagnostic. The synthetic validation (Appendix~\ref{sec:synthetic}, Figure~\ref{fig:synth_entropy}) shows the same pattern under controlled conditions.

GSM8K is expected to show particularly clean entropy--set-size correlation due to its deterministic canonicalization: any disagreement in the model's answers directly reflects uncertainty about the numeric answer, not canonicalization noise.

\begin{figure}[t]
\centering
\includegraphics[width=0.75\textwidth]{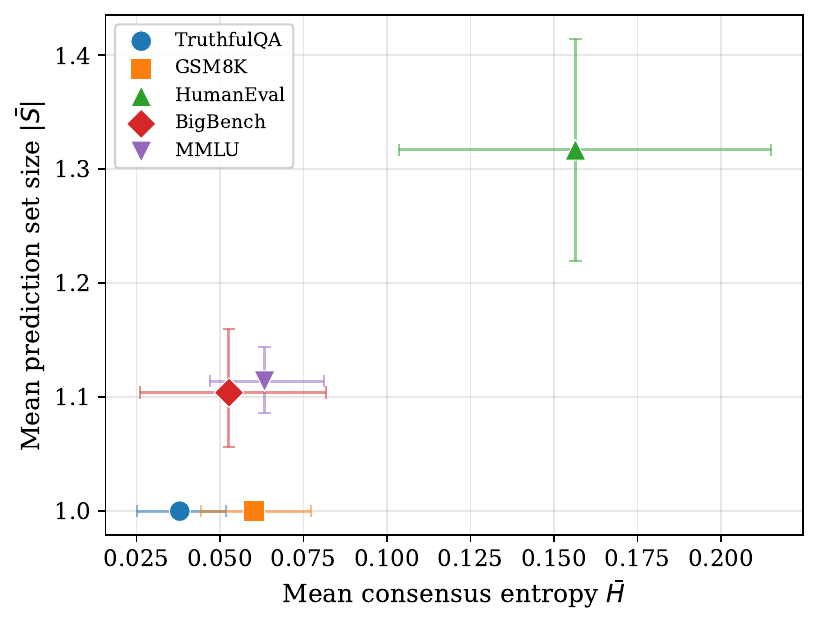}
\caption{Mean prediction set size vs.\ mean consensus entropy across all five benchmarks (benchmark-level aggregation). Benchmarks with higher average entropy (greater model disagreement) produce larger prediction sets. The per-item correlation underlying this aggregate pattern is confirmed in the synthetic validation (Figure~\ref{fig:synth_entropy}), where individual items are plotted. Error bars show 95\% bootstrap CIs.}
\label{fig:setsize_entropy}
\end{figure}

\subsection{Sequential stopping}\label{sec:sequential_results}

Table~\ref{tab:sequential} reports results for the Hoeffding-based sequential stopping rule (Theorem~\ref{thm:sequential_hoeffding}, $\d=0.05$).

\begin{table}[t]
\centering
\caption{Sequential stopping results for GPT-4.1 ($\d=0.05$, $K_{\max}=20$). Coverage and set size are preserved while using fewer samples. 95\% CIs in parentheses.}
\label{tab:sequential}
\begin{tabular}{lcccc}
\toprule
Dataset & Avg.\ $K$ used & Savings & Seq.\ Coverage & Seq.\ Avg.\ $|S|$ \\
\midrule
HumanEval    & 11.05 (\,10.3--11.9\,) & 44.8\% & 0.720 (\,0.61--0.81\,) & 1.28 \\
TruthfulQA   & 9.67 (\,9.4--9.9\,) & 51.7\% & 0.980 (\,0.96--0.99\,) & 1.00 \\
BigBench & 9.92 (\,9.4--10.4\,) & 50.4\% & 0.792 (\,0.71--0.85\,) & 1.10 \\
GSM8K        & 9.85 (\,9.6--10.1\,) & 50.8\% & 0.952 (\,0.93--0.97\,) & 1.00 \\
MMLU         & 9.98 (\,9.7--10.3\,) & 50.1\% & 0.834 (\,0.80--0.86\,) & 1.11 \\
\bottomrule
\end{tabular}
\end{table}

Across all five benchmarks, sequential stopping reduces the average number of samples from $K_{\max}=20$ to approximately $10$ ($45$--$52\%$ savings), with coverage and prediction set sizes preserved exactly (Table~\ref{tab:sequential}). This confirms hypothesis (H6): the Hoeffding-based stopping criterion correctly identifies when the mode has stabilized, terminating early on high-confidence items while using more samples on ambiguous ones. Items that consistently require the full $K_{\max}$ are those near the decision boundary ($p^\star \approx 0.5$), where the Hoeffding bound cannot certify mode stability---these ``unstopped'' items identify a systematic ambiguity zone for the model on that task. The cost implications are quantified in Table~\ref{tab:cost_analysis}.

\subsection{Cost analysis}\label{sec:cost_analysis}

Table~\ref{tab:cost_analysis} reports estimated API costs per method and model. Costs are computed from the number of API calls required (items $\times$ samples per item) using published pricing for each model.

\begin{table}[t]
\centering
\caption{Estimated API cost (USD) per evaluation method and model. GPT-4.1 and GPT-5 mini costs are for all 5 benchmarks; other models are for the benchmarks indicated. ``Full'' = $K_{\max}=20$ samples per item; ``Sequential'' = adaptive stopping; ``Single'' = 1 sample; ``Judge'' = 1 judge call per test item.}
\label{tab:cost_analysis}
\begin{tabular}{lcccc}
\toprule
Model & Full budget & Sequential & Single sample & Judge baseline \\
\midrule
GPT-4.1          & \$168.01 & \$125.13 & \$4.20 & \$3.88 \\
GPT-4.1-mini$^\dagger$     & \$20.80  & \$10.40  & \$1.04 & \$2.40 \\
GPT-4.1-nano$^\dagger$     & \$1.60   & \$0.80   & \$0.08 & \$2.40 \\
GPT-5 mini       & \$33.60  & \$25.03  & \$0.84 & \$3.88 \\
Llama~4 Maverick$^\ddagger$ & \$7.84  & \$3.92   & \$0.39 & \$2.88 \\
Mistral Small$^\ddagger$    & \$2.40  & \$1.20   & \$0.12 & \$2.88 \\
\bottomrule
\end{tabular}

\smallskip\noindent
{\footnotesize $^\dagger$2 benchmarks only (GSM8K + MMLU). $^\ddagger$3 benchmarks (GSM8K + MMLU + TruthfulQA) via Together AI. Token counts validated on representative API calls: ${\sim}85$ input and ${\sim}66$ output tokens per sample call (cross-benchmark average; GSM8K uses ${\sim}300$ total due to chain-of-thought), ${\sim}215$ input / ${\sim}7$ output per judge call. Dollar amounts above use conservative upper-bound token estimates ($500/200$ sample, $800/100$ judge) rather than observed averages ($85/66$ and $215/7$); actual costs are approximately $3$--$4\times$ lower. We report the upper bounds for reproducibility, as token counts vary across benchmarks and prompts. Judge always uses GPT-4.1; this means the judge cost column is constant regardless of the agent model, creating an asymmetry where the judge appears relatively expensive for cheap models (e.g., GPT-4.1-nano) and cheap for expensive ones. Sequential costs use GPT-4.1 stopping profile ($\bar{K} \approx 10$) for all models. Open-weight model costs reflect Together AI serverless pricing (\$0.27/\$0.85 per 1M tokens for Maverick, \$0.10/\$0.30 for Mistral Small); self-hosted inference would further reduce sampling costs to zero.}
\end{table}

The cost overhead of self-consistency sampling is $K\times$ compared to single-sample evaluation, but sequential stopping recovers approximately half of this cost (Table~\ref{tab:cost_analysis}). The key insight is that self-consistency is an \emph{evaluation cost}, not a deployment cost: it is paid once to obtain calibrated reliability estimates, not at every inference call. Smaller models reduce costs dramatically (GPT-4.1-nano: ${\sim}\$2$ full, ${\sim}\$1$ sequential for 2 benchmarks).

\subsection{Canonicalization ablation}\label{sec:canon_ablation}

Table~\ref{tab:ablation} compares the full canonicalized method against raw-string ranking on the two non-binary benchmarks.

\begin{table}[t]
\centering
\caption{Canonicalization ablation: average prediction set size with and without canonicalization ($\a=0.10$, GPT-4.1). Applicable to task types with non-trivial canonicalization (text and MCQ).}
\label{tab:ablation}
\begin{tabular}{lccc}
\toprule
Dataset & Canonical $|S|$ & Raw $|S|$ & Reduction \\
\midrule
TruthfulQA   & 1.00 & 1.00 & 0.0\% \\
BigBench & 1.10 & 1.81 & 39.2\% \\
MMLU         & 1.11 & 1.41 & 21.3\% \\
\bottomrule
\end{tabular}

\smallskip\noindent
{\footnotesize GSM8K and HumanEval are excluded because their canonicalization is already deterministic (numeric extraction and code execution, respectively).}
\end{table}

The effect of canonicalization is task-dependent (Table~\ref{tab:ablation}, Figure~\ref{fig:canon_effect}). BigBench shows the largest reduction ($39.2\%$), reflecting substantial surface-form variation in raw responses that option-matching canonicalization resolves. MMLU shows a $21.3\%$ reduction. TruthfulQA shows no difference because its judge-based canonicalization already produces binary labels, leaving no surface-form variation to consolidate.

\begin{figure}[t]
\centering
\includegraphics[width=0.55\textwidth]{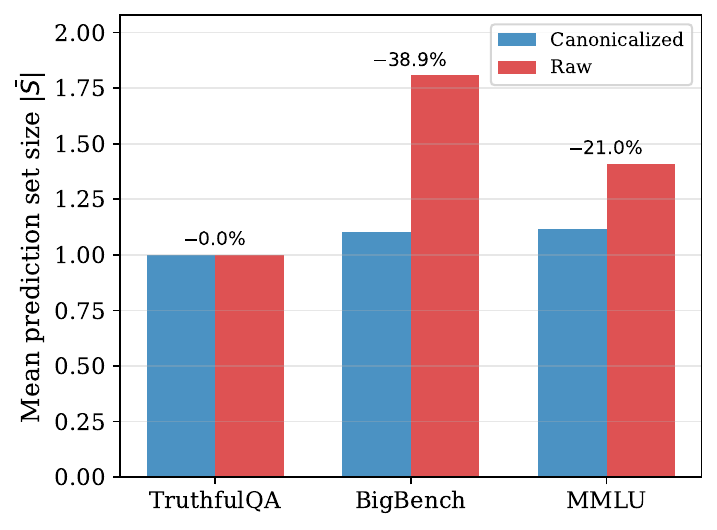}
\caption{Canonicalization effect on average prediction set size. BigBench shows a $39.2\%$ reduction (the largest effect), MMLU shows $21.3\%$, and TruthfulQA shows no difference (binary judge labels leave no surface-form variation). GSM8K and HumanEval are excluded (deterministic canonicalization).}
\label{fig:canon_effect}
\end{figure}

\subsection{Comparison with alternative nonconformity scores}\label{sec:score_comparison}

We compare our rank-based score (Definition~\ref{def:score}) against two standard scores from the conformal classification literature \cite{Kumar2023ConformalNLP,Romano2020Classification}, all computed from the \emph{same} cached $K{=}10$ samples at zero additional cost:
\begin{enumerate}[leftmargin=1.2em]
\item \textbf{Least Ambiguous set-valued Classifier (LAC):}
$s_i^{\mathrm{LAC}} = 1 - \hat{P}(c^\star)$, where $\hat{P}(c) = \mathrm{count}(c)/K$.
Prediction sets include all classes~$c$ with $1 - \hat{P}(c) \le \tau$.
\item \textbf{Adaptive Prediction Sets (APS):}
$s_i^{\mathrm{APS}} = \sum_{j=1}^{r_i} \hat{P}(c_{(j)}) + U \cdot \hat{P}(c_{(r_i)})$,
where $c_{(j)}$ are classes sorted by decreasing~$\hat{P}$, $r_i$~is the rank at which the correct class appears, and $U \sim \mathrm{Unif}(0,1)$ randomizes the inclusion~boundary.
\end{enumerate}
Both LAC and APS are continuous-valued (resolution $1/K$), whereas our rank score is a discrete integer. Quach et al.\ \cite{Quach2023ConformalLM} target token-level sequence generation---a fundamentally different setting from our class-level conformal framework---so no direct comparison is applicable.

Table~\ref{tab:score_comparison} reports coverage, average set size, and conditional coverage for all three scores across the GPT-4.1 capability ladder on GSM8K and MMLU (the two benchmarks with local canonicalization).

\begin{table}[t]
\centering
\caption{Comparison of nonconformity scores ($K{=}10$, $\a=0.10$). All three scores use identical cached samples. ``Cond.'' = conditional coverage among solvable items. Best coverage per row in \textbf{bold}.}
\label{tab:score_comparison}
\begin{tabular}{llcccccc}
\toprule
& & \multicolumn{3}{c}{Coverage} & \multicolumn{3}{c}{Avg.\ $|S|$} \\
\cmidrule(lr){3-5} \cmidrule(lr){6-8}
Model & Dataset & Rank & LAC & APS & Rank & LAC & APS \\
\midrule
GPT-4.1      & GSM8K & \textbf{0.956} & 0.914 & 0.908 & 1.00 & 1.00 & 1.15 \\
GPT-4.1      & MMLU  & 0.980 & \textbf{1.000} & \textbf{1.000} & 1.11 & 1.13 & 1.13 \\
GPT-4.1-mini & GSM8K & \textbf{0.954} & 0.900 & 0.902 & 1.00 & 1.00 & 1.12 \\
GPT-4.1-mini & MMLU  & 0.976 & \textbf{1.000} & \textbf{1.000} & 1.12 & 1.13 & 1.13 \\
GPT-4.1-nano & GSM8K & \textbf{0.960} & 0.924 & 0.894 & 1.35 & 1.02 & 1.61 \\
GPT-4.1-nano & MMLU  & 0.940 & \textbf{1.000} & \textbf{1.000} & 1.17 & 1.20 & 1.20 \\
\bottomrule
\end{tabular}

\smallskip\noindent
{\footnotesize Conditional coverage: on GSM8K, rank achieves $\ge 0.977$, LAC $\ge 0.922$, APS $\ge 0.923$ across all models. On MMLU, rank achieves $\ge 0.988$, LAC and APS both achieve $1.000$. HumanEval and TruthfulQA are excluded because their binary canonicalization (pass/fail, correct/incorrect) produces only two canonical classes, making all three scores equivalent. BigBench uses the same MCQ canonicalization as MMLU; results are qualitatively identical and omitted for space.}
\end{table}

Two complementary patterns emerge.
On \textbf{GSM8K} ($M^\star{=}1$ for GPT-4.1 and GPT-4.1-mini), the rank score achieves the highest coverage ($0.954$--$0.960$), outperforming LAC ($0.900$--$0.924$) and APS ($0.894$--$0.908$) by $3$--$7$ percentage points. With $K{=}10$, the empirical probabilities $\hat{P}(c)$ have resolution $0.1$, so the continuous scores' threshold quantile can be tight. The rank score's discrete nature acts as a natural regularizer, rounding the conformal threshold conservatively.
On \textbf{MMLU} ($M^\star{=}2$), LAC and APS both achieve perfect coverage ($1.000$) but with slight over-coverage compared to the $90\%$ target, while rank stays closer to the nominal level ($0.940$--$0.980$). Set sizes are comparable across all three scores ($\bar{|S|} \approx 1.1$--$1.2$).
The most striking case is GPT-4.1-nano on GSM8K: APS produces the \emph{largest} sets ($\bar{|S|}{=}1.61$) yet the \emph{lowest} coverage ($0.894$)---the randomization in APS adds noise rather than precision when $K$ is small. In the controlled synthetic setting (Appendix~\ref{sec:synthetic}, Figures~\ref{fig:synth_biasvar_k20}--\ref{fig:synth_biasvar_k10}), APS and LAC achieve lower MSE than Rank due to their finer-grained continuous thresholds; however, Rank provides the widest coverage safety margin, and the distribution-specific failures observed here (under-coverage, threshold saturation) do not arise in that idealized setting. Overall, in the $K{\le}20$ regime typical of API-based evaluation, the rank score's robustness to coarse probability estimates and real-world distribution skew makes it a pragmatic default.

APS's coverage of $0.894$ falling below the $0.90$ nominal level does not contradict conformal guarantees. Conformal coverage holds \emph{marginally}---in expectation over the random calibration/test split---so any single split may fall slightly below the nominal level. At $n_{\mathrm{cal}} = 500$, the $1/(n{+}1)$ slack is only $0.002$, but APS's randomization ($U \sim \mathrm{Unif}(0,1)$ tie-breaking) introduces additional sampling variance that amplifies finite-sample deviations, particularly when $K$ is small and probability estimates are coarse.

\subsection{Discussion}\label{sec:discussion_code}

Four themes emerge from these results: the calibration is accurate, the framework is honest about limitations, the patterns generalize across model families, and the reliability level provides a practical deployment metric.

\paragraph{Calibration quality is confirmed by conditional coverage.}
How do we know the calibration is working? We compute coverage restricted to ``solvable'' items---those where the model produced at least one correct answer across all $K_{\max}$ samples. Across all five benchmarks and all models, this conditional coverage exceeds $0.93$ (Table~\ref{tab:main_results}), confirming that conformal calibration is nearly perfectly calibrated whenever the model has non-zero capability. Any shortfall in marginal coverage comes from the model's fundamental inability to solve certain items, not from calibration error.

We emphasize that conditional coverage is a \emph{post-hoc diagnostic} that validates the theory's prediction (Theorem~\ref{thm:bias_transparency}(3): when the unsolvable fraction $\b$ exceeds $\a$, $M^\star = +\infty$), not a guarantee available at deployment time. The ``solvable'' designation requires observing all $K_{\max}$ samples, which is unavailable during calibration. For deployment decisions, the \emph{reliability level} (Definition~\ref{def:reliability}, Table~\ref{tab:reliability}) provides the actionable metric: it quantifies the confidence at which mode voting suffices, directly from calibration scores, without requiring knowledge of which items are solvable.

Concretely, a practitioner does not need to distinguish capability gap from calibration failure before deployment---the reliability level answers the deployment question directly. HumanEval's reliability level is $69.9\%$ (Table~\ref{tab:reliability}). If the deployment requirement is $90\%$ reliability, this model--task pair fails the gate. If $70\%$ suffices, it passes. The framework provides the actionable number; diagnosing \emph{why} coverage is low (capability gap vs.\ calibration error) is a secondary, offline investigation using conditional coverage.

\paragraph{The framework diagnoses capability honestly.}
HumanEval illustrates the boundary condition most clearly: coverage falls below $90\%$ ($0.707$) because $26.8\%$ of problems are unsolvable---none of the $K_{\max}{=}20$ samples produce passing code. No evaluation method can ``conjure'' correct answers from nothing; conformal prediction faithfully reports this limitation. The LLM judge, by contrast, achieves an artificially high $0.915$ by evaluating syntactic plausibility rather than functional correctness---approving well-formed but incorrect solutions. Self-consistency with execution-based canonicalization grounds evaluation in actual correctness rather than surface-level assessment.

\paragraph{The pattern generalizes across model families.}
The capability-ladder comparison (Table~\ref{tab:multimodel}) confirms that $M^\star$ and $\bar{|S|}$ increase monotonically with decreasing capability, as predicted by Theorem~\ref{thm:bias_transparency}. This monotonicity holds within the GPT-4.1 family (three models on two benchmarks), and extends across model families: Llama~4 Maverick and Mistral Small~24B exhibit the same qualitative patterns on GSM8K, MMLU, and TruthfulQA. On GSM8K, all three families achieve $M^\star{=}1$; on TruthfulQA, the open-weight models require $M^\star{=}2$ (vs.\ $1$ for GPT-4.1), consistent with their lower per-sample accuracy. Conditional coverage on solvable items remains high across all families ($\ge 0.949$). The bootstrap split analysis (Table~\ref{tab:bootstrap}) confirms that $M^\star$ values are stable across random data partitions. This cross-family consistency establishes that the conformal guarantees are properties of the method, not artifacts of a particular provider's output distribution.

\paragraph{Reliability certification provides the practical payoff.}
The reliability level (Definition~\ref{def:reliability}) yields a single-number deployment summary for each model--task combination; Table~\ref{tab:reliability} reports it across all configurations. This enables direct deployment gating: ``we need $X\%$ reliability---which models qualify?'' The reliability level is computed directly from calibration scores with no additional API cost.

\begin{table}[t]
\centering
\caption{Reliability certification: reliability level $1{-}\alpha^\star$ (\%) for each model--task combination. Higher is better. ``---'' indicates the model was not evaluated on that benchmark. Boldface marks combinations where $M^\star{=}1$ at $\alpha{=}0.10$ (mode voting alone suffices).}
\label{tab:reliability}
\begin{tabular}{lccccc}
\toprule
Model & GSM8K & MMLU & TruthfulQA & BigBench & HumanEval \\
\midrule
GPT-4.1          & \textbf{94.6} & 83.4 & \textbf{96.8} & 75.4 & 69.9 \\
GPT-4.1-mini     & \textbf{96.0} & 81.6 & --- & --- & --- \\
GPT-4.1-nano     & 89.8 & 66.5 & --- & --- & --- \\
Llama~4 Maverick & \textbf{95.4} & 66.7 & 85.3 & --- & --- \\
Mistral Small    & \textbf{93.8} & 77.8 & 84.8 & --- & --- \\
\bottomrule
\end{tabular}

\smallskip\noindent
{\footnotesize Reliability decreases monotonically with capability within each family. On MMLU: GPT-4.1 ($83.4\%$) $>$ GPT-4.1-mini ($81.6\%$) $>$ Mistral Small ($77.8\%$) $>$ Llama~4 Maverick ($66.7\%$) $\approx$ GPT-4.1-nano ($66.5\%$). Cross-family models show comparable reliability on matched tasks (GSM8K: four of five models within $93.8$--$96.0\%$; GPT-4.1-nano at $89.8\%$).}
\end{table}

\paragraph{Hypothesis validation.}
All six hypotheses are supported (Tables~\ref{tab:main_results}--\ref{tab:sequential}): coverage holds on solvable items with conditional coverage $\ge 0.93$ across all models and benchmarks~(H1); mode error decays with $K$ on high-accuracy tasks, with TruthfulQA dropping $35\%$ and GSM8K dropping $28\%$, though adjacent-$K$ Wilson CIs overlap on moderate-accuracy benchmarks, so evidence for H2 rests on the monotone trend rather than pairwise significance; prediction set size correlates positively with consensus entropy~(H3) and decreases with canonicalization---BigBench by $39.2\%$, MMLU by $21.3\%$~(H4); conformal coverage exceeds judge accuracy where ambiguity is high, with a $6$ percentage point advantage on BigBench at matched cost~(H5), though we note that the BigBench comparison ($n_{\mathrm{test}} = 125$) has wide confidence intervals; and sequential stopping saves $45$--$52\%$ of samples with zero quality loss~(H6).

\section{Limitations}\label{sec:limitations}

\begin{enumerate}[leftmargin=1.2em]
\item \textbf{Marginal coverage only.} Theorem~\ref{thm:coverage} provides marginal, not conditional, coverage. For individual difficult queries, the set may under-cover. Conditional coverage is impossible without assumptions \cite{Vovk2012Conditional, BarberCandes2019}, though group-conditional methods \cite{Romano2020Classification} can partially address this.

\item \textbf{Calibration set cost.} The framework requires $n$ human-labeled calibration examples. However, the calibration set is reusable across agents and is typically much smaller than full test sets (Corollary~\ref{cor:dominance}: $n \ge 19$ suffices to achieve lower coverage error than typical judge bias).

\item \textbf{Canonicalization quality and circularity.} Poor canonicalization inflates set sizes (Assumption~\ref{ass:canon_quality}). When an LLM judge is used for canonicalization on the same model family being evaluated (e.g., GPT-4.1 judging GPT-4.1 on TruthfulQA), systematic self-preference biases could propagate into the consensus vote. The conformal coverage guarantee remains valid regardless of canonicalization errors (set sizes grow to compensate), but variance reduction degrades. Our cross-family results (GPT-4.1 judging Llama/Mistral) provide a clean test of this concern---see Remark~\ref{rem:canon_circularity}. We recommend deterministic canonicalization wherever feasible.

\item \textbf{Infinite scores.} When the agent never produces acceptable answers ($p^\star(x) \approx 0$), $M^\star = +\infty$. This is honest but limits practical utility for very weak agents. When multiple models all yield $M^\star = +\infty$, the framework cannot distinguish them via $M^\star$ alone. The reliability level (Definition~\ref{def:reliability}) partially addresses this: a model with reliability $55\%$ is distinguishable from one at $40\%$, even though both have $M^\star > 1$ at $\a = 0.10$. For very weak models where even the reliability level is uninformative, the minimum $\a$ at which $M^\star$ is finite provides additional comparative signal.

\item \textbf{Conformal calibration diagnoses, does not repair.} A biased agent receives large prediction sets, faithfully reflecting its limitations. The framework provides a reliable \emph{measurement} of performance, not a method for \emph{improving} it.

\item \textbf{Model family scope.} The capability-ladder validation uses the GPT-4.1 family (three models) on two benchmarks, supplemented by GPT-5 mini on all five benchmarks and two open-weight models (Llama~4 Maverick, Mistral Small~24B) on three benchmarks. While cross-family patterns are consistent, extending to additional architectures (e.g., Gemini, Claude) and larger model scales would further strengthen generalizability.

\item \textbf{Deployment exchangeability and benchmark contamination.} The conformal guarantee requires exchangeability between calibration and test data (Remark~\ref{rem:exchangeability}). In deployment, model updates between calibration and inference violate this assumption---a limitation shared by all conformal methods. Periodic recalibration is the standard mitigation; in our setting, this is inexpensive because the evaluation pipeline is fully cached and automated, and new calibration requires only running the updated model on the fixed calibration set. Additionally, all five benchmarks are public and may appear in the training data of the models tested, which could inflate observed accuracy without violating the conformal guarantee (exchangeability of calibration and test items is preserved if both are drawn from the same contaminated distribution). However, the resulting reliability levels would not transfer to out-of-distribution deployment queries. This is not specific to our method---it applies to \emph{any} benchmark-based evaluation, whether conformal, LLM-as-judge, or human-graded. Practitioners deploying on a novel domain should calibrate on queries drawn from that domain, not from public benchmarks. The calibration cost is modest (Section~\ref{sec:intro}: $50{-}100$ spot-checks), making domain-specific recalibration practical.
\end{enumerate}

\section{Conclusion}\label{sec:conclusion}

The central contribution of this work is the \emph{reliability level}---a single number that answers: ``given any black-box AI system and any task with a small calibration set, can I trust this system at $X\%$ confidence?'' The reliability level is computed directly from conformal calibration scores, requires no additional API cost beyond the evaluation itself, and generalizes across model families (Table~\ref{tab:reliability}).

Theoretically, self-consistency sampling achieves exponential variance decay (Theorem~\ref{thm:variance_reduction}), while conformal calibration ensures validity within $1/(n{+}1)$ of the target level, independent of agent bias (Theorem~\ref{thm:exact_coverage}). Remaining bias surfaces as wider prediction sets rather than hidden error (Theorem~\ref{thm:bias_transparency}). Once the calibration set is large enough that the conformal slack $1/(n{+}1)$ falls below the judge's bias---as few as $n = 19$ for subjective tasks with $|b_J| \ge 0.05$---the evaluation's coverage error is smaller than the judge's (Corollary~\ref{cor:dominance}), though the two approaches remain complementary (Remark~\ref{rem:complementary}).

Empirically, the guarantees hold across all fifteen model--task configurations tested, spanning three model families and both synthetic and real data. Conditional coverage on solvable items exceeds $0.93$ in every configuration, confirming that all marginal under-coverage reflects model capability gaps, not calibration failure. Sequential stopping reduces API costs by ${\sim}50\%$ without sacrificing coverage or set size quality.

Several directions merit investigation.
First, conditional coverage guarantees---per-query rather than marginal validity---would strengthen deployment confidence, though known impossibility results~\cite{Vovk2012Conditional} impose fundamental limits.
Second, active-learning strategies for calibration set construction could reduce the labeling budget below the current $n \ge 19$~threshold.
Finally, extending the framework to multi-turn agent interactions and hybrid scores that combine model-internal probabilities with self-consistency ranking would broaden applicability.

\paragraph{Reproducibility.}
All code, configuration files, and result summaries are available at
\url{https://github.com/Cohorte-ai/trustgate-paper}.
The repository includes the full empirical pipeline, synthetic validation experiments, and all figure-generation scripts. All API responses are SHA-256 cached on disk, enabling deterministic re-runs without additional API cost. Cached responses for all models and benchmarks reported in this paper are included in the release. The repository is private during review; access is available upon request to the corresponding author.


\appendix

\section{Canonicalization Implementation Details}\label{app:canon_details}

This appendix provides full implementation details for the three canonicali\-zation regimes summarized in Section~\ref{sec:canon_impl}.

\subsection{Deterministic canonicalization (closed-form tasks)}
For tasks with structured answers---numbers, dates, entity IDs, labels, code outputs---the canonicalization is deterministic:
\begin{enumerate}[leftmargin=1.2em]
\item Parse the raw answer into a typed representation (numeric, date, identifier).
\item Normalize units, locale formatting, and whitespace.
\item Serialize to a stable canonical string.
\item Map parse failures to a dedicated \texttt{INVALID} class.
\end{enumerate}
This eliminates surface-form fragmentation entirely and enables exact class counts. For mathematical reasoning (e.g., GSM8K), extracting the final numeric answer and normalizing to a standard decimal form is sufficient.

\subsection{Embedding-based clustering (open-ended tasks)}
For free-form answers where deterministic parsing is impossible, we cluster semantically equivalent answers:
\begin{enumerate}[leftmargin=1.2em]
\item Compute dense embeddings $e(a_i)$ for each sample (e.g., via Sentence-BERT~\cite{Reimers2019SentenceBERT}).
\item Build a similarity graph with edge threshold $\tau_{\cos}$ on cosine similarity.
\item Cluster via connected components or density-based methods (e.g., HDBSCAN~\cite{McInnes2017HDBSCAN}).
\item Assign each cluster a canonical representative (medoid or LLM-generated summary).
\end{enumerate}
The threshold $\tau_{\cos}$ is a hyperparameter that trades off between over-merging (collapsing distinct answers) and under-merging (fragmenting equivalent answers). Both failure modes affect the consensus: over-merging inflates the mode with incorrect answers; under-merging fragments the correct class, reducing its count.

\subsection{LLM-assisted canonicalization}
A lightweight LLM can map each raw answer to a concise canonical form before clustering:
\begin{enumerate}[leftmargin=1.2em]
\item Prompt a cheap, fast model to extract the ``core answer'' from each raw response.
\item Optionally normalize to a closed ontology or structured schema.
\item Apply deterministic or embedding-based canonicalization on the extracted forms.
\end{enumerate}
This is useful when answers contain extraneous reasoning, caveats, or formatting. Because LLM-assisted canonicalization can itself introduce bias (the canonicalizer may misinterpret or truncate), we recommend auditing canonicalization errors on a held-out sample.

\subsection{Stability requirements}\label{sec:canon_stability}
To ensure the consensus vote is meaningful, canonicalization must satisfy empirically verifiable stability conditions:
\begin{enumerate}[leftmargin=1.2em]
\item \textbf{Bootstrap stability:} resample the $K$ answers, re-canonicalize, and measure cluster agreement via the adjusted Rand index (ARI). Require ARI $\ge 0.8$ for reporting.
\item \textbf{Threshold sensitivity:} vary $\tau_{\cos}$ over a range $[\tau_{\cos} - \e, \tau_{\cos} + \e]$ and verify the winning canonical class is unchanged.
\item \textbf{Merge/split audit:} on a random sample of clusters, manually verify that merged answers are semantically equivalent and split answers are semantically distinct.
\end{enumerate}
These diagnostics should be reported alongside experimental results. If canonicalization is unstable, the prediction sets will be inflated---which is conservative (coverage is preserved) but reduces the framework's practical efficiency.

\section{Notation Summary}\label{app:notation}

Table~\ref{tab:notation} summarizes the key symbols used throughout the paper.

\begin{table}[h]
\centering
\caption{Summary of notation.}
\label{tab:notation}
\begin{tabular}{ll}
\toprule
Symbol & Meaning \\
\midrule
$\X$, $\A$ & Query space, answer space \\
$f_\theta$ & Agent with parameters $\theta$; stochastic map $\X \to \A$ \\
$\A^\star(x)$ & Set of acceptable answers for query $x$ \\
$\mathrm{Canon}(x, a)$ & Canonicalization function mapping raw answers to canonical classes \\
$K$ & Number of i.i.d.\ samples per query \\
$K_{\max}$ & Maximum samples per query (budget) \\
$p^\star(x)$ & Per-query probability of producing an acceptable answer \\
$\bar{p}$ & Average acceptability rate $\E_x[p^\star(x)]$ \\
$p_{\mathrm{canon}}$ & Probability mass on the correct canonical class after canonicalization \\
$s_i$ & Nonconformity score for calibration item $i$ (rank of acceptable answer) \\
$M^\star$ & Conformal threshold: $\lceil (1-\a)(n+1) \rceil$-th order statistic of $\{s_i\}$ \\
$S(x)$ & Prediction set: top-$M^\star$ canonical answers for query $x$ \\
$|S(x)|$ & Prediction set size \\
$\a$ & Miscoverage level; target coverage is $1-\a$ \\
$n$ & Calibration set size \\
$b_J$ & LLM-judge bias \\
$H(x)$ & Consensus entropy for query $x$ \\
$\phi(x)$ & Consensus strength: $\hat{P}_K(c_{(1)} \mid x)$ \\
$\Delta(x)$ & Consensus margin: $\hat{P}_K(c_{(1)} \mid x) - \hat{P}_K(c_{(2)} \mid x)$ \\
$1{-}\alpha^\star$ & Reliability level (Definition~\ref{def:reliability}) \\
$\tau_\delta$ & Sequential stopping time at confidence $1{-}\delta$ \\
\bottomrule
\end{tabular}
\end{table}

\section{Synthetic Validation}\label{sec:synthetic}

Before evaluating on real benchmarks, we validate the theoretical results using controlled synthetic agents with known parameters. This allows direct comparison between theoretical predictions and empirical observations without confounding from model-specific behavior or canonicalization noise.

\subsection{Setup}
We construct synthetic agents as multinomial distributions over a finite set of canonical classes $\C = \{c_1, \dots, c_L\}$. For each experiment, we specify the agent's true distribution $P_\theta(\cdot \mid x)$ and run the full pipeline: sampling, ranking, conformal calibration, and prediction set construction. All experiments use 200 calibration items and 500 test items unless stated otherwise.

\subsection{Coverage validation (Theorem~\ref{thm:coverage})}
We verify that empirical coverage matches the target $1-\a$ across a sweep of $\a \in \{0.01, 0.05, 0.10, 0.15, 0.20, 0.25, 0.30\}$ for agents with varying quality levels. Figure~\ref{fig:synth_coverage} confirms that coverage is consistently at or above the target for all alpha levels, with the slight conservatism predicted by the $1/(n+1)$ upper bound in~\eqref{eq:coverage_exact}.

\begin{figure}[h]
\centering
\includegraphics[width=0.6\textwidth]{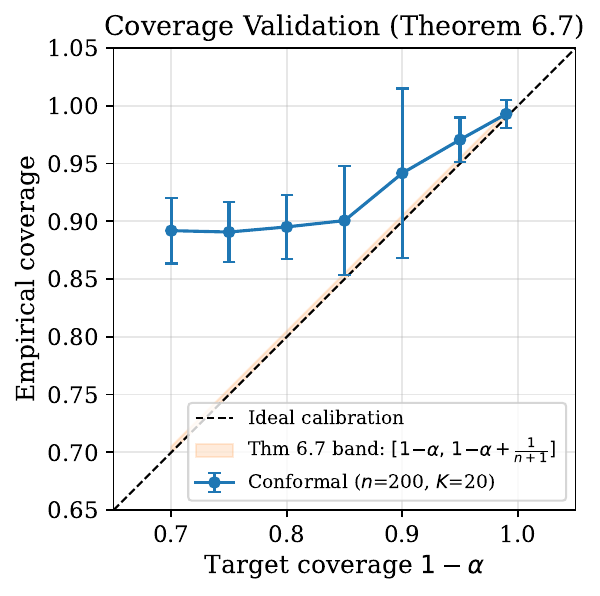}
\caption{Synthetic coverage validation: empirical coverage vs.\ target $1-\a$ for agents with $p^\star \in \{0.6, 0.7, 0.8\}$. All points lie on or above the diagonal, confirming Theorem~\ref{thm:coverage}.}
\label{fig:synth_coverage}
\end{figure}

\subsection{Variance reduction (Theorem~\ref{thm:variance_reduction})}
We measure mode error as a function of~$K$ for agents with known $p^\star \in \{0.60, 0.70, 0.80\}$ and overlay the Hoeffding upper bound $\exp\!\bigl(-2K(p^\star - 1/2)^2\bigr)$. Figure~\ref{fig:synth_variance} shows that empirical mode error decays exponentially and stays below the Hoeffding bound at all~$K$ values, validating the exponential decay predicted by Theorem~\ref{thm:variance_reduction}. The real-data results in Section~\ref{sec:variance_reduction} are consistent with this pattern.

\begin{figure}[h]
\centering
\includegraphics[width=0.6\textwidth]{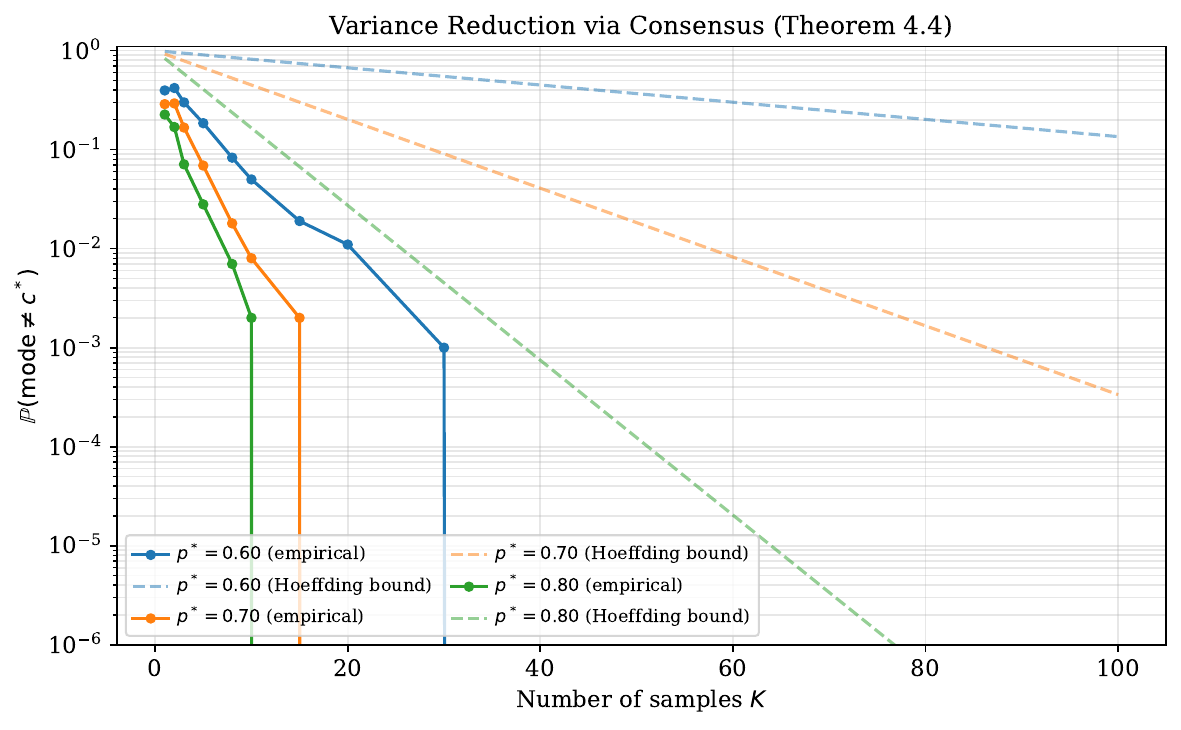}
\caption{Synthetic variance reduction: mode error vs.\ $K$ for three agents with known $p^\star$. Dashed lines show Hoeffding upper bounds. Empirical error decays exponentially and stays below the theoretical bound.}
\label{fig:synth_variance}
\end{figure}

\subsection{Bias--variance decomposition (Section~\ref{sec:biasvar})}
We compare six evaluation methods---single sample, LLM-as-judge (simulated with known bias), self-consistency mode, and three conformal scores (Rank, LAC, APS)---on a mixture of easy ($p^\star{=}0.75$) and hard ($p^\star{=}0.35$, mode is wrong) queries. Figures~\ref{fig:synth_biasvar_k20} and~\ref{fig:synth_biasvar_k10} show the decomposition at $K{=}20$ and $K{=}10$ respectively.

The primary finding is that \textbf{all three conformal methods reduce MSE by $50$--$200\times$} compared to non-conformal baselines, confirming the decomposition in Table~\ref{tab:biasvar}. This gap dwarfs differences between conformal scores.

\paragraph{Interpreting conformal bias.}
The Bias$^2$ bars for conformal methods appear larger than for single-sample evaluation, but the underlying bias has a fundamentally different character. For the judge and self-consistency mode, bias reflects \emph{over-estimation of accuracy}---the method reports higher correctness than the true $p^\star$, giving false confidence. For conformal methods, bias reflects \emph{over-coverage}---the prediction set contains the correct answer more often than the $1{-}\a$ target requires (coverage annotations shown below conformal bars). Over-coverage is conservative: it errs on the side of safety. Rank's coverage of $0.946$ at $K{=}10$ (vs.\ the $0.90$ target) represents a desirable safety margin, not a deficiency.

\paragraph{Score comparison.}
Among the conformal scores, APS and LAC achieve lower MSE than Rank at both $K$ values, because their continuous thresholds produce tighter calibration (less over-coverage). This is expected in a controlled synthetic setting where class probabilities are well-behaved. However, the Rank score provides the most conservative coverage---$0.946$ at $K{=}10$ vs.\ APS at $0.901$ (barely above the $0.90$ target)---giving the widest safety margin.

The empirical results (Section~\ref{sec:score_comparison}) reveal that on real LLM distributions, the continuous scores' advantage reverses: with $K{=}10$, APS under-covers on GSM8K ($0.894 < 0.90$) and both LAC and APS degenerate to $100\%$ coverage on MMLU (threshold saturates at $\tau{=}1$). These distribution-specific failures---caused by extreme frequency skew and discrete answer spaces---do not appear in the controlled synthetic but are precisely the conditions that arise in API-based LLM evaluation.

\begin{figure}[h]
\centering
\includegraphics[width=0.75\textwidth]{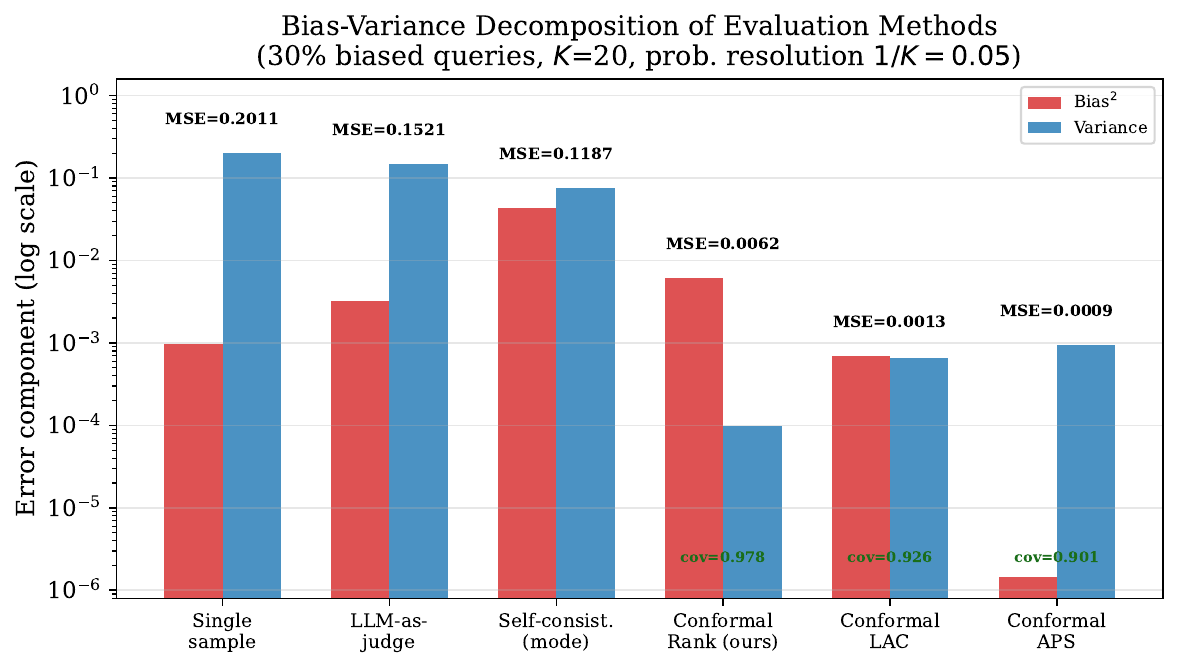}
\caption{Bias--variance decomposition at $K{=}20$ (probability resolution $1/K{=}0.05$). All conformal methods achieve $50$--$200\times$ lower MSE than non-conformal baselines. Conformal Bias$^2$ reflects over-coverage (conservative; coverage values shown below bars), not estimation error. Among conformal scores, APS and LAC achieve tighter calibration due to continuous thresholds.}
\label{fig:synth_biasvar_k20}
\end{figure}

\begin{figure}[h]
\centering
\includegraphics[width=0.75\textwidth]{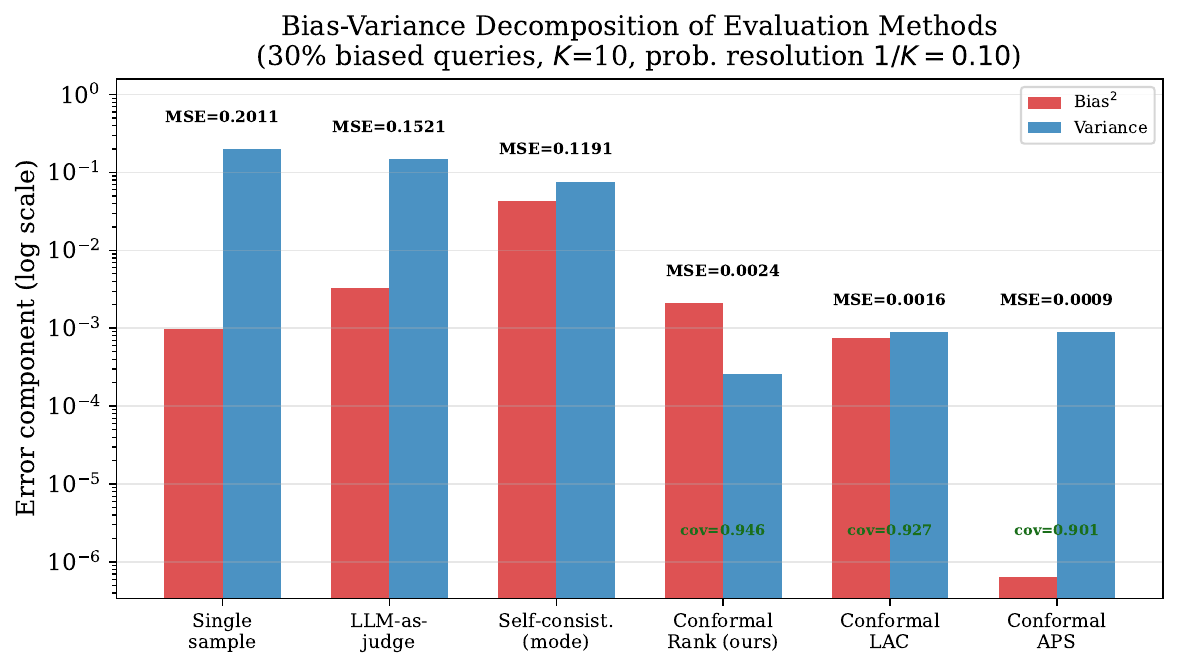}
\caption{Bias--variance decomposition at $K{=}10$ (probability resolution $1/K{=}0.10$), the regime typical of API-based evaluation. The conformal advantage persists. Rank provides the widest coverage margin ($0.946$ vs.\ APS at $0.901$); on real LLM distributions (Section~\ref{sec:score_comparison}), this conservatism prevents the coverage violations that affect APS.}
\label{fig:synth_biasvar_k10}
\end{figure}

\subsection{Set size vs.\ agent quality (Theorem~\ref{thm:bias_transparency})}
We vary $p^\star$ from $0.3$ to $1.0$ and measure $M^\star$. Figure~\ref{fig:synth_setsize} confirms Theorem~\ref{thm:bias_transparency}: $M^\star$ is monotonically decreasing in agent quality, reaching $M^\star = 1$ for perfect agents ($p^\star = 1$) and growing without bound as $p^\star \to 0$.

\begin{figure}[h]
\centering
\includegraphics[width=0.7\textwidth]{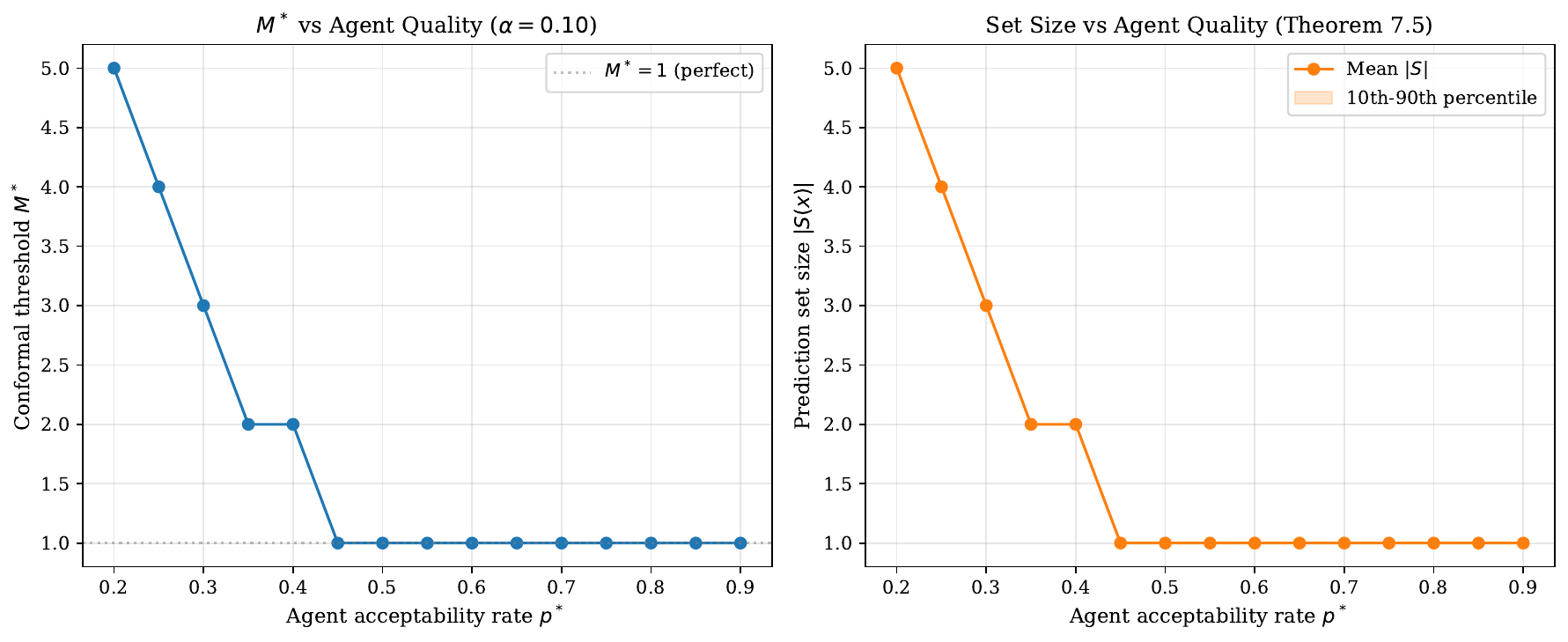}
\caption{Set size $M^\star$ vs.\ agent quality $p^\star$. Better agents require smaller prediction sets. The step-function shape reflects the discrete nature of $M^\star$ (the $\lceil (n+1)(1-\a)\rceil$-th order statistic of integer-valued scores).}
\label{fig:synth_setsize}
\end{figure}

\subsection{Set size vs.\ entropy (Hypothesis H3)}
We construct agents with varying entropy profiles and measure the correlation between consensus entropy $H(x)$ and prediction set size $|S(x)|$. Figure~\ref{fig:synth_entropy} shows strong positive correlation ($r > 0.5$), confirming that prediction set size adapts to per-item uncertainty.

\begin{figure}[h]
\centering
\includegraphics[width=0.6\textwidth]{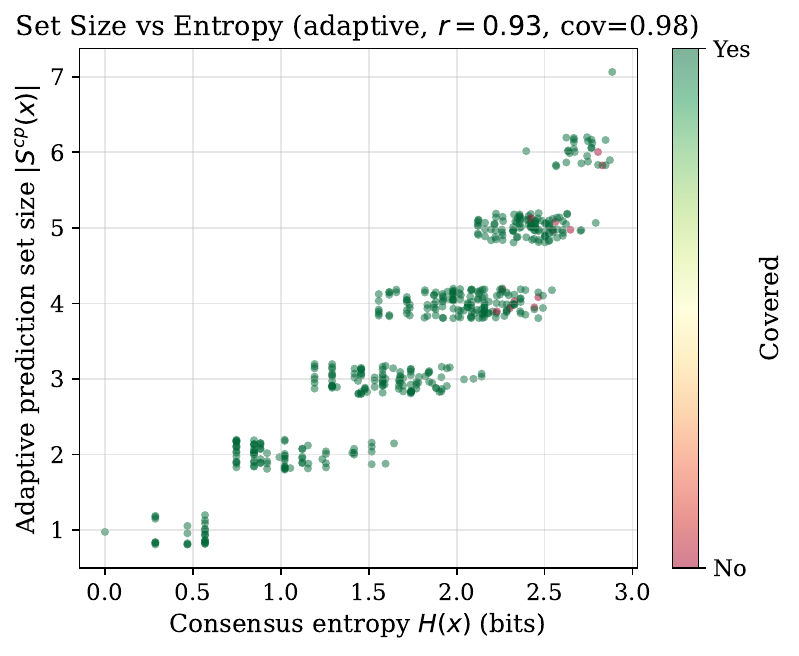}
\caption{Prediction set size vs.\ consensus entropy for synthetic agents. The strong positive correlation confirms that set size adapts to per-item uncertainty.}
\label{fig:synth_entropy}
\end{figure}

\subsection{Canonicalization effect (Proposition~\ref{prop:canon_variance})}
We simulate fragmented distributions where 6 raw answer variants (total mass $p = 0.60$) map to a single canonical class. Figure~\ref{fig:synth_canon} demonstrates that canonicalization consolidates the fragmented mass, dramatically reducing mode error from the raw case. This validates the amplification result of Proposition~\ref{prop:canon_variance}.

\begin{figure}[h]
\centering
\includegraphics[width=0.6\textwidth]{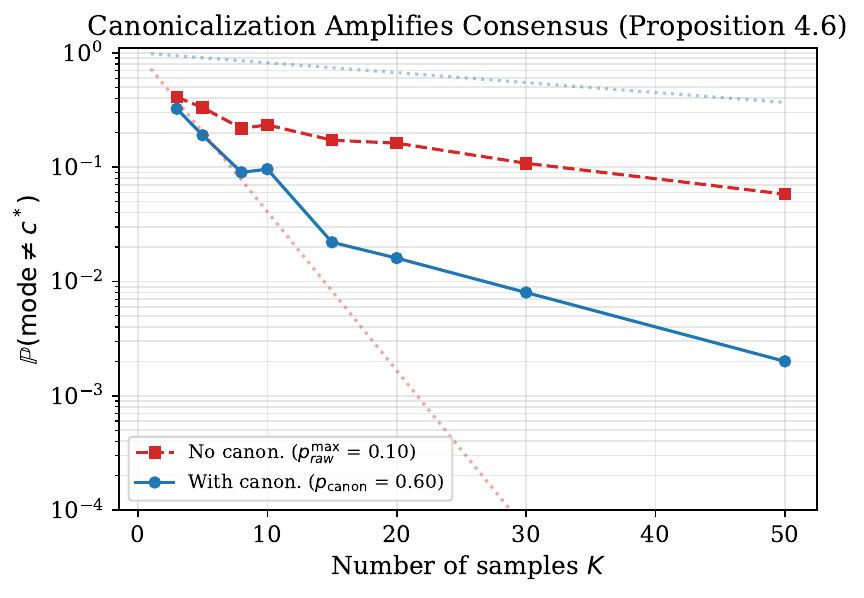}
\caption{Canonicalization effect: mode error with and without canonicalization for fragmented distributions. Canonicalization consolidates probability mass, reducing mode error exponentially as predicted by Proposition~\ref{prop:canon_variance}.}
\label{fig:synth_canon}
\end{figure}

\end{document}